\documentclass[11pt,letterpaper]{article}

 \usepackage[hscale=0.75,vscale=0.8]{geometry}

\usepackage{todonotes}

\usepackage[utf8]{inputenc} 
\usepackage[T1]{fontenc}    
\usepackage{hyperref}       
\usepackage{url}            
\usepackage{booktabs}       
\usepackage{amsfonts}       
\usepackage{nicefrac}       
\usepackage{microtype}      
\usepackage{color,xcolor}         

\usepackage{amsmath,amsthm,bbm}
\usepackage{amssymb}
\usepackage{url}
\usepackage{graphicx}
\usepackage{float}
\usepackage{wrapfig}
\usepackage{hyperref}

\newcommand{\Lm}{\mathrm{L}^\alpha}

\newtheorem{theorem}{Theorem}
\newtheorem*{theorem*}{Theorem}

\newtheorem{corollary}[theorem]{Corollary}
\newtheorem{definition}[theorem]{Definition}

\newtheorem{lemma}[theorem]{Lemma}
\newtheorem*{lemma*}{Lemma}

\newtheorem{proposition}[theorem]{Proposition}

\newtheorem{assumption}[theorem]{Assumption}

\usepackage{comment}

\usepackage{enumitem}

\newcommand{\imagi}{\mathsf{i}}

\allowdisplaybreaks[4]

\usepackage{cleveref}

\usepackage[square,sort,comma,numbers]{natbib}

\title{Algorithmic Stability of Heavy-Tailed SGD with General Loss Functions} 

\author{Anant Raj$^*$ \\
 Coordinated Science Laboraotry \\
 University of Illinois Urbana-Champaign. \\
  Inria, Ecole Normale Sup\'erieure \\
  PSL Research University, Paris, France. \\
  \texttt{anant.raj@inria.fr} 
  \vspace{.25cm}
 \\
 \and Lingjiong Zhu$^*$ \\
  Department of Mathematics \\
  Florida State University, FL, USA. \\
  \texttt{zhu@math.fsu.edu} \\
  \and Mert Gürbüzbalaban \\
  Department of Management\\ Science and Information Systems \\
 Rutgers University, Piscataway, USA. \\ 
 Princeton University, NJ, USA\\
  \texttt{mg1625@princeton.edu}  \\
  \and Umut \c{S}im\c{s}ekli \\
  Inria, CNRS, Ecole Normale Sup\'erieure \\
  PSL Research University, Paris, France. \\
  \texttt{umut.simsekli@inria.fr} \\
}

\begin{document}

\maketitle
\def\thefootnote{*}\footnotetext{These authors contributed equally to this work.}\def\thefootnote{\arabic{footnote}}

\begin{abstract}
Heavy-tail phenomena in stochastic gradient descent (SGD) have been reported in several empirical studies.  Experimental evidence in previous works suggests a strong interplay between the heaviness of the tails and generalization behavior of SGD. To address this empirical phenomena theoretically, several works have made strong topological
and statistical assumptions to link the generalization error to heavy tails. Very recently, new generalization bounds have been proven, indicating a non-monotonic relationship between the generalization error and heavy tails, which is more pertinent to the reported empirical observations. While these bounds do not require additional topological assumptions given that SGD can be modeled using a heavy-tailed stochastic differential equation (SDE), they can only apply to simple quadratic problems. 
In this paper, we build on this line of research and develop generalization bounds for a more general class of objective functions, which includes non-convex functions as well. Our approach is based on developing Wasserstein stability bounds for heavy-tailed SDEs and their discretizations, which we then convert to generalization bounds. Our results do not require any nontrivial assumptions; yet, they shed more light to the empirical observations, thanks to the generality of the loss functions.

\end{abstract}

\section{Introduction} \label{sec:intro}
Many supervised learning problems can be expressed as an instance of the \emph{risk minimization problem}
\begin{align}
    \label{eqn:pop_risk}
    \min_{\theta\in\mathbb{R}^d} \left\{ F(\theta):= \mathbb{E}_{x\sim \mathcal{D}} [f(\theta,x)] \right\},
\end{align}
where $x \in \mathcal{X}$ is a random data point, distributed according to an unknown probability distribution $\mathcal{D}$ and taking values in the data space $\mathcal{X}$, $\theta$ denotes the parameter vector of the model to be learned and $f(\theta, x)$ is the instantaneous loss of misprediction with parameters $\theta$ corresponding to the data point $x$. With different choices of the function $f$, we can recover many problems in supervised learning from deep learning to logistic regression or support vector machines \citep{shalev2014understanding}. 

As $\mathcal{D}$ is unknown in many scenarios, directly attacking \eqref{eqn:pop_risk} is often not possible. Assuming we have access to a training dataset $X_n = \{x_1, \ldots , x_n\} \subset \mathcal{X}^n$ with $n$ independent and identically distributed (i.i.d.) observations, in practice, we can consider the \emph{empirical risk minimization} (ERM) problem instead, given as follows: 
%
%
\begin{align*}
\min_{\theta \in \mathbb{R}^d} \left\{ \hat{F}(\theta,X_n) := \frac1{n} \sum_{i=1}^n f(\theta, x_i) \right\}. 
\end{align*}

One of the most popular algorithms for attacking the ERM problem is stochastic gradient descent (SGD) that is based on the following recursion:
\begin{align}
    \theta_{k+1} = \theta_k - \eta \nabla \tilde{F}_{k+1}(\theta_k,X_{n}), \label{eqn:sgd}
\end{align}
where $\eta$ is the step-size (or learning-rate) and $$\nabla \tilde{F}_{k}(\theta,X) := \frac1{b} \sum_{i\in \Omega_k} \nabla f(\theta, x_i)$$ 
is the stochastic gradient, with $\Omega_k \subset \{1,\dots,n\}$ being a random subset drawn with or without replacement, and $b := |\Omega_k| \ll n$ being the batch-size. 

Understanding the generalization properties of SGD has been a major challenge in modern machine learning. In this context, the goal is to bound the so-called generalization error: $|\hat{F}(\theta,X_n) - F(\theta)|$, either in expectation or in high probability. 

While a plethora of approaches have been proposed to address this task \citep{cao2019generalization,lei2020fine,neu2021information,park2022generalization}, a promising approach among those has been based on the theoretical and empirical observations which showed that SGD can exhibit a \emph{heavy-tailed} behavior, depending on the choice of hyperparameters ($\eta$ and $b$), the data distribution $\mathcal{D}$, and the geometry of the loss function $f$ \citep{gurbuzbalaban2020heavy,hodgkinson2021multiplicative}. This has motivated the use of `heavy-tailed proxies' for SGD, which --to some extent-- facilitated the analysis of SGD in terms of its generalization error. Examples of such proxies include gradient descent with additive heavy-tailed noise:
\begin{align}
\label{eqn:sgd_addht}
    \theta_{k+1} = \theta_k - \eta \nabla \hat{F}(\theta_k, X_n) + \xi_{k+1},
\end{align}
where $(\xi_k)_{k\geq 1}$ is a sequence of heavy-tailed random vectors, potentially with unbounded higher-order moment, i.e., $\mathbb{E}\|\xi_k\|^p = +\infty$ for some $p>1$ (see e.g., \citep{nguyen2019first,zhang2020adam,wang2021convergence}).




Another popular proxy for heavy-tailed SGD is based on a \emph{continuous-time} version of \eqref{eqn:sgd_addht}, which is expressed by the following stochastic differential equation (SDE):
\begin{equation}
\label{eqn:levysde}
d\theta_{t}=-\nabla\hat{F}(\theta_{t},X_{n})dt+\sigma d\Lm_{t},
\end{equation}
where $\sigma >0$ is a scale parameter, $\Lm_{t}$ is a $d$-dimensional $\alpha$-stable L\'{e}vy process, which has heavy-tailed increments and will be formally defined in the next section\footnote{This type of SDEs have also received some attention in terms of limits of deterministic gradient descent with dynamical regularization \citep{lim2022chaotic}. }, and $\alpha \in (0,2]$ denotes the `tail-exponent' such that as $\alpha$ gets smaller the process $\Lm_{t}$ becomes heavier-tailed.

Within this mathematical framework, \citet{simsekli2020hausdorff} proved an upper-bound (which was then improved in \citep{hodgkinson2021generalization}) for the worst-case generalization error over the trajectories of \eqref{eqn:levysde}. The bound informally reads as follows: with probability at least $1-\delta$, it holds that
\begin{align*}
    \sup_{\theta \in \Theta} \left|\hat{F}(\theta,X_n) - F(\theta)\right| \lesssim \sqrt{\frac{\alpha + I(\Theta,X_n) + \log(1/\delta)}{n} },
\end{align*}
where $\Theta$ denotes the trajectory of \eqref{eqn:levysde}, i.e., 
\begin{equation*}
\Theta := \left\{\theta \in \mathbb{R}^d: \exists t \in [0,1], \theta = \theta_t \right\}, 
\end{equation*}
with $\theta_t$ being the solution of \eqref{eqn:levysde}, and $I(\Theta,X_n)$ denotes a form of `mutual information' between the trajectory $\Theta$ and the data sample $X_n$ (cf.\ \citep{xu2017information}). This result suggests that the generalization error is essentially determined by two terms: (i) the tail exponent $\alpha$, as the tails get heavier the generalization error will be lower, (ii) the statistical dependency between the trajectory and the data sample, the lower the dependency the better the generalization performance.  

While these results illuminated an interesting connection between heavy-tails and generalization, they unfortunately rely on nontrivial topological assumptions on $\Theta$ and the mutual information term cannot be controlled in an interpretable way in general. On the other hand, \citet{barsbey2021heavy} empirically illustrated that the relation between the tail exponent and the generalization error might not be monotonic in practical applications; an observation which cannot be directly supported by the bound in \citep{simsekli2020hausdorff}  and \citep{hodgkinson2021generalization}.

Aiming to alleviate these issues, very recently, \citet{Raj2022} considered the same problem from the lens of algorithmic stability \citep{bousquet2002stability,hardt2016train}. They considered the SDE \eqref{eqn:levysde} and further simplified it by choosing the loss function as a simple quadratic, i.e., $f(\theta,x) = (\theta^\top x)^2$. They showed that any parameter vector $\theta$ provided by \eqref{eqn:levysde} (or its Euler-Maruyama discretization with small enough small step-size) cannot be algorithmically stable. However, when the algorithmic stability is measured by a \emph{surrogate loss function} instead (reminiscent of \citep{wang2021convergence}), the parameter vector $\theta$ becomes algorithmically stable, which immediately implies generalization. Their bound further illustrated that the relation between $\alpha$ and the generalization error might not be monotonic, which is in line with the observations provided in \citep{barsbey2021heavy}.

While the bounds in \cite{Raj2022} do not require additional topological assumptions and do not contain a mutual information term as opposed to \citep{simsekli2020hausdorff,hodgkinson2021generalization}, their analysis technique heavily relies on the fact that $f$ is a quadratic, hence cannot be directly extended beyond quadratic loss functions.

In this paper, we aim at filling this gap and prove algorithmic stability bounds the SDE \eqref{eqn:levysde} (and its Euler-Maruyama discretization) with general loss functions, which can be even non-convex. Our contributions are as follows:
\begin{itemize}[topsep=0pt,leftmargin=.11in]
    \item We first focus on the continuous-time setting and prove Wasserstein stability bounds for two SDEs of the form of \eqref{eqn:levysde} with different drift functions. Our results cover both the finite-time case, i.e., $t<\infty$ and the stationary case, i.e., $t \to \infty$. We build upon recently introduced stochastic analysis tools for uniform-in-time Wasserstein error bounds for Euler-Maruyama discretization \citep{chen2022approximation} to obtain a novel Wasserstein stability bound for two $\alpha$-stable L\'{e}vy-driven SDEs. Our analysis relies on an additional pseudo-Lipschitz like condition for the underlying process and the dataset (Assumption~\ref{assump:1}) and careful adaption of the tools in \citep{chen2022approximation} to our context (Lemma~\ref{lem:key} and Theorem~\ref{thm:W:1:re} in the Appendix) as well as additional analysis (Lemma~\ref{C:0:formula}) that allows us to characterize the dependence of our bounds on the tail-index $\alpha$. Our derived bounds would be interesting on their own to a much broader scope.   
    \item By following \citep{raginsky2016information}, we translate the derived Wasserstein stability bounds to algorithmic stability bounds. Similar to \citep{Raj2022}, our approach necessitates surrogate loss functions to measure algorithmic stability. Our results reveal that the relation between heaviness of the tail $\alpha$ and the generalization error might not be monotonic, indicating that the conclusions of \citep{Raj2022} extends to the general case. 
    \item By 
    combining our results with \citep{chen2022approximation}, we extend our bounds to the Euler-Maruyama discretization of \eqref{eqn:levysde} (that is of the form of \eqref{eqn:sgd_addht}) and show that for small enough step-sizes the discrete-time process achieves almost identical stability bounds.
\end{itemize}
Contrary to \citep{simsekli2020hausdorff,hodgkinson2021generalization,lim2022chaotic}, our bounds do not rely on any topological regularity assumptions and they further do not contain a mutual information term. Moreover, our results shed more light to the non-monotonic relation between heavy tails and the generalization error, as empirically observed in \citep{barsbey2021heavy,Raj2022}, since they are applicable to  non-convex losses, as opposed to \citep{Raj2022}. We also note that our generalization bounds and Wasserstein bounds are independent of time. Such a result was previously shown in \citep{farghly2021time} in the context of Brownian-motion driven SDEs and their discretizations, our work uses different techniques considering Levy-driven SDEs and studies the link between the generalization and the coefficient of heavy tail.



\section{Notations and Technical Background}
\label{sec:bg}

\textbf{Gradients and Hessians.}
For any twice continuously differentiable function $f:\mathbb{R}^{d}\rightarrow\mathbb{R}$, 
we denote by $\nabla f$ and $\nabla^{2}f$ the gradient and the Hessian
of $f$. First-order and second-order directional derivatives of $f$ are defined as
\begin{align}
\nabla_{v}f(x)&:=\lim_{\epsilon\rightarrow 0}\frac{f(x+\epsilon v)-f(x)}{\epsilon}, \notag \\
\nabla_{v_{2}}\nabla_{v_{1}}f(x)&:=\lim_{\epsilon\rightarrow 0}\frac{\nabla_{v_{1}}f(x+\epsilon v_{2})-\nabla_{v_{1}}f(x)}{\epsilon},
\end{align} 
for any directions $ v, v_1, v_2 \in \mathbb{R}^d$. If $f$ is three times continuously differentiable, then third-order derivatives along the directions $v_1, v_2$ are given by
\begin{align}
    \nabla_{v_{2}}\nabla_{v_{1}}\nabla f(x)&:=\lim_{\epsilon\rightarrow 0}\frac{\nabla_{v_{1}} \nabla f(x+\epsilon v_{2})-\nabla_{v_{1}}\nabla f(x)}{\epsilon}.
\ 
\end{align}

\textbf{Wasserstein distance.}
For $p\geq 1$, the $p$-Wasserstein distance between two probability measures $\mu$ and $\nu$ on $\mathbb{R}^{d}$
is defined as \citep{villani2008optimal}:
\begin{equation}
\mathcal{W}_{p}(\mu,\nu)=\left\{\inf\mathbb{E}\Vert X-Y\Vert^p\right\}^{1/p},
\end{equation}
where the infimum is taken over all coupling of $X\sim\mu$ and $Y\sim\nu$.
In particular, the $1$-Wasserstein distance has the following dual representation \citep{villani2008optimal}:
\begin{equation}
\mathcal{W}_{1}(\mu,\nu)=\sup_{h\in\text{Lip}(1)}\left|\int_{\mathbb{R}^{d}}h(x)\mu(dx)-\int_{\mathbb{R}^{d}}h(x)\nu(dx)\right|,
\end{equation}
where $\text{Lip}(1)$ consists of the functions $h:\mathbb{R}^{d}\rightarrow\mathbb{R}$
that are $1$-Lipschitz.

\textbf{Algorithmic stability.} Algorithmic stability is an important notion in learning theory, which has pave the way for several important theoretical results \cite{bousquet2002stability,hardt2016train}. Let us first state the notion of algorithmic stability as defined in \citep{hardt2016train}.

\begin{definition}[\citet{hardt2016train}, Definition 2.1] \label{def:stability}
 For a (surrogate) loss function $\ell:\mathbb{R}^d \times \mathcal{X} \rightarrow \mathbb{R}$, an algorithm $\mathcal{A} : \bigcup_{n=1}^\infty \mathcal{X}^n \to \mathbb{R}^d$ is $\varepsilon$-uniformly stable if
 \begin{align}
     \sup_{X\cong \hat{X}}\sup_{z \in \mathcal{X}}~ \mathbb{E}\left[\ell(\mathcal{A}(X),z) - \ell(\mathcal{A}(\hat{X}),z) \right] \leq \varepsilon ,
 \end{align}
where the first supremum is taken over data $X, \hat{X} \in \mathcal{X}^n$   that differ by one element, denoted by $X\cong \hat{X}$.
\end{definition}
Here, we intentionally use a different notation for the loss $\ell$ (as opposed to $f$), as our theory will require the algorithmic stability to be measured by using a surrogate loss function, which might be different than the original loss $f$. 

We now provide a result from \citep{hardt2016train} which relates algorithmic stability to the generalization performance of a randomized algorithm. Before stating the result, similar to $\hat{F}$ and $F$, we respectively define the empirical and population risks with respect to the loss $\ell$ as follows:
\begin{align*}
    \hat{R}(w,X_n) :=& \frac{1}{n}\sum_{i=1}^n \ell(w,x_i), \quad 
    R(w) := \mathbb{E}_{x\sim \mathcal{D}} [\ell(w,x)].
\end{align*}

\begin{theorem}[\citet{hardt2016train}, Theorem 2.2]
Suppose that $\mathcal{A}$ is an $\varepsilon$-uniformly stable algorithm, then the expected generalization error is bounded by
\begin{align}
    \left|\mathbb{E}_{\mathcal{A},X_n}~\left[ \hat{R}(\mathcal{A}(X_n),X_n) \right] - R(\mathcal{A}(X_n))  \right| \leq \varepsilon.
\end{align}

\end{theorem}



\textbf{Alpha-stable distributions.}
A scalar random variable $X$ is said to follow a symmetric $\alpha$-stable distribution, denoted
by $X\sim\mathcal{S}\alpha\mathcal{S}(\sigma)$, if its characteristic function
takes the form:
$\mathbb{E}\left[e^{\imagi uX}\right]=\exp\left(-\sigma^{\alpha}|u|^{\alpha}\right)$, for any $u\in\mathbb{R}$,
where $\sigma>0$ is known as the scale
parameter that measures the spread
of $X$ around $0$ and for $\alpha\in(0,2]$ which is known as the tail-index
that determines
the tail thickness of the distribution. The tail becomes heavier as $\alpha$ gets smaller.
The $\alpha$-stable distribution $\mathcal{S}\alpha\mathcal{S}$ appears as the limiting distribution
in the generalized central limit theorems for
a sum of i.i.d.\ random variables
with infinite variance \citep{paul1937theorie}. 
The probability density function of a symmetric $\alpha$-stable distribution, $\alpha\in(0,2]$,
does not yield closed-form expression in general 
except for a few special cases; 
for example $\mathcal{S}\alpha\mathcal{S}$ reduces to the Cauchy and the Gaussian distributions, respectively, when $\alpha=1$ and $\alpha=2$.
When $0<\alpha<2$, the moments
are finite only up to the order $\alpha$
in the sense that 
$\mathbb{E}[|X|^{p}]<\infty$
if and only if $p<\alpha$,
which implies infinite variance.

Finally, $\alpha$-stable distribution can be extended
to the high-dimensional case for random vectors. 
One of the most commonly used extension is the rotationally symmetric $\alpha$-stable distribution.
$X$ follows a $d$-dimensional rotationally symmetric $\alpha$-stable distribution
if it admits the characteristic function $\mathbb{E}\left[e^{\imagi \langle u,X\rangle}\right]=e^{-\sigma^{\alpha}\Vert u\Vert_{2}^{\alpha}}$ for
any $u\in\mathbb{R}^{d}$. For further details of $\alpha$-stable distributions, we refer to \citep{ST1994}.

\textbf{L\'{e}vy processes.}
L\'{e}vy processes are stochastic processes with independent and stationary increments.
Their successive displacements can be viewed as the continuous-time
analogue of random walks.
L\'{e}vy processes include the Poisson process, the Brownian motion,
the Cauchy process, and more generally stable
processes; see e.g. \citep{bertoin1996,ST1994,Applebaum}.
L\'{e}vy processes in general admit jumps
and have heavy tails which are appealing
in many applications; see e.g. \citep{Cont2004}.

In this paper, we will consider the rotationally symmetric $\alpha$-stable L\'{e}vy process, denoted by $\Lm_{t}$ in $\mathbb{R}^d$ and is defined as follows.
\begin{itemize}[topsep=0pt,leftmargin=.11in]
\item $\Lm_0=0$ almost surely;
\item For any $t_{0}<t_{1}<\cdots<t_{N}$, the increments $\Lm_{t_{n}}-\Lm_{t_{n-1}}$
are independent;
\item The difference $\Lm_{t}-\Lm_{s}$ and $\Lm_{t-s}$
have the same distribution, with the characteristic function $\exp(- (t-s)^\alpha\|u\|_2^\alpha)$ for $t>s$;
\item $\Lm_{t}$ has stochastically continuous sample paths, i.e.
for any $\delta>0$ and $s\geq 0$, $\mathbb{P}(\|\Lm_{t}-\Lm_{s}\|>\delta)\rightarrow 0$
as $t\rightarrow s$.
\end{itemize}
When $\alpha=2$, $\Lm_{t}=\sqrt{2}\mathrm{B}_{t}$, where $\mathrm{B}_{t}$ is the standard $d$-dimensional Brownian motion.

\section{Main Results}

In this section, we present our main theoretical results. To ease the notation, we will consider the following SDE in lieu of \eqref{eqn:levysde}:
\begin{equation}\label{eq:sde_invar_1}
d\theta_{t}=-\nabla\hat{F}(\theta_{t},X_{n})dt+d\Lm_{t},
\end{equation}
in the rest of the paper\footnote{
Note that the stationary distribution of \eqref{eqn:levysde}
is the same as the stationary distribution of
$d\theta_{t}=-\sigma^{-\alpha}\nabla\hat{F}(\theta_{t},X_{n})dt+d\Lm_{t}$,
and so that we can easily adapt
our main result (Theorem~\ref{thm:W:1})
to the general case $\sigma>0$.}.

Our road map is as follows. We will first consider the continuous-time case \eqref{eq:sde_invar_1}, i.e., we will set the learning algorithm as $\mathcal{A}(X_n) = \theta_t$ for some $t \in [0, +\infty]$, where $\theta_\infty$ denotes a sample from the stationary distribution of the SDE \eqref{eq:sde_invar_1}. As our aim is to prove algorithmic stability bounds for this choice of algorithm, we then consider another dataset $\hat{X}_n \cong X_n$, which differ from $X_n$ by one element, accordingly define the following SDE:
%
%
\begin{equation}\label{eq:sde_invar_2}
d\hat{\theta}_{t}=-\nabla\hat{F}(\hat{\theta}_{t},\hat{X}_{n})dt+d\Lm_{t},
\end{equation}
such that $\mathcal{A}(\hat{X}_n) = \hat{\theta}_t$. 
Then, we will argue that, for any time $t$, the laws of $\theta_t$ and $\hat{\theta}_t$ will be close to each other in the $1$-Wasserstein metric. 
%

By considering a surrogate loss function $\ell$, which we will assume to be $\mathcal{L}$-Lipschitz, 
our bound on the Wasserstein distance between $\mbox{Law}(\theta_t)$ and $\mbox{Law}(\hat{\theta}_t)$ (Theorem~\ref{thm:W:1}) will immediately provide us a generalization bound thanks to the dual representation of the $1$-Wasserstein distance (cf.\ \citep[Lemma 3]{raginsky2016information}):
\begin{align}
    &\left|\mathbb{E}_{\theta_t,X_n}~\left[ \hat{R}(\theta_t,X_n) \right] -  R(\theta_t)  \right| \nonumber  \\
    &\qquad\qquad \leq\mathcal{L} \sup_{X_n\cong \hat{X}_n}   \mathcal{W}_{1}\left(\mbox{Law}(\theta_t),\mbox{Law}(\hat{\theta}_t)\right),
    \label{eqn:wass_stab}
\end{align}
where $\mbox{Law}(\theta_t)$ and $\mbox{Law}(\hat{\theta}_t)$ respectively depend on $X_n$ and $\hat{X}_n$ due to the form of the SDEs. The reason why we require a surrogate loss function is the fact that we need the Lipschitz continuity of the loss to be able to derive the bound in \eqref{eqn:wass_stab}. However, as we will detail in the next subsection, our assumptions on the true loss $f$ will be incompatible with the Lipschitz continuity of $f$.

After proving a generalization bound of the form \eqref{eqn:wass_stab}, we will further investigate the behavior of the bound with respect to the heaviness of the tail which is characterized by the tail-index $\alpha$. Finally, we will consider the discrete-time case, where we will show that almost identical results hold for the Euler-Maruyama discretizations of \eqref{eq:sde_invar_1} and \eqref{eq:sde_invar_2}, as long as a sufficiently small step-size is chosen. 

\subsection{Assumptions}

In this section we state our main assumptions and we will assume that they hold throughout the paper. For any $w\in\mathbb{R}^{d}$, we use $\theta_{t}^{w}$ 
to denote the process $\theta_{t}$ that starts
at $\theta_{0}=w$.

\begin{assumption}\label{assump:1}
For every $x \in \mathcal{X}$, there exist universal constants $K_1$, $K_2$ such that
\begin{align}
&\|\nabla f(\theta, x) - \nabla f(\hat{\theta},\hat{x})\| 
\leq  K_1 \|\theta- \hat{\theta}\| + K_2 \| x-\hat{x} \|( \|\theta \| + \|\hat{\theta}\| +1) .
\end{align}
\end{assumption}

This assumption is a pseudo-Lipschitz like condition (similar to the one in \cite{erdogdu}) on the loss $f$.
Under this assumption, for two datasets $X_{n}$ and $\hat{X}_{n}$, we immediately have the following property:
\begin{align}
\left\Vert \nabla\hat{F}(\theta,X_{n})-\nabla\hat{F}\left(\hat{\theta},\hat{X}_{n}\right)\right\Vert
\leq
K_{1}\Vert\theta-\hat{\theta}\Vert 
+\rho(X_{n},\hat{X}_{n})K_{2}\left(\Vert\theta\Vert+\Vert\hat{\theta}\Vert+1\right),
\end{align}
where
\begin{equation}
\rho(X_n,\hat{X}_n) := \frac1{n}\sum_{i=1}^n \| x_i-\hat{x}_i \|.    
\end{equation}
We will show that the term $\rho(X_n,\hat{X}_n)$ will have an important role in terms of Wasserstein stability. 

By following \citep{chen2022approximation}, we also make
the following assumption.

\begin{assumption}\label{assump:2}
For every $x \in \mathcal{X}$, $f(\cdot,x)$ is 
three-times continuously differentiable, and for any $\theta_{1},\theta_{2}\in\mathbb{R}^{d}$, there exist universal constants $B$, $m$, $K$, $L$, and $M$ such that
\begin{align*}
&\Vert\nabla f(0,x)\Vert\leq B,
\\
&\left\langle\nabla f(\theta_{1},x)-\nabla f(\theta_{2},x),\theta_{1}-\theta_{2}\right\rangle
\leq
-m\Vert\theta_{1}-\theta_{2}\Vert^{2}+K,
\end{align*}
and
\begin{align*}
&\left\Vert\nabla_{v}\nabla f(\theta,x)\right\Vert
\leq L\Vert v\Vert,
\\
&\left\Vert\nabla_{v_{1}}\nabla_{v_{2}}\nabla f(\theta,x)\right\Vert
\leq
M\Vert v_{1}\Vert\Vert v_{2}\Vert,
\end{align*}
for any $v,v_1, v_2\in\mathbb{R}^{d}$.
\end{assumption}
The first part of this assumption is 
common
in stochastic analysis and often referred to as dissipativity \citep{raginsky2017non,gao2022global}. The second part of the assumption amounts to requiring the drift $\nabla f(x)$ to have bounded third-order directional derivatives (see also \citep{chen2022approximation}). This would be satisfied for instance if $f$ has bounded third-order derivatives on the set $\mathcal{X}$.

\subsection{Continuous-Time Dynamics}


Now, we are ready to state our first theorem
that characterizes the $1$-Wasserstein distance between
$\theta_t$ and $\hat{\theta_t}$ at any finite time $t$,
which is uniform in $t$. As a result, we also obtain an upper-bound on 
the $1$-Wasserstein distance between
the unique invariant distribution $\mu$ of $(\theta_{t})_{t\geq 0}$
and the unique invariant distribution $\hat{\mu}$ of $(\hat{\theta}_{t})_{t\geq 0}$.\footnote{Here, we know that under our assumptions by the results of \citep{chen2022approximation}, invariant distributions $\mu$ and $\hat{\mu}$ exist.} 

The full statement of the theorem is rather lengthy and is given in the Section~{\ref{ap:proof_theorem_W_1}} in the Appendix. For clarity, in the next theorem, we provide our upper-bound on the distance between the invariant distributions, i.e., $t\to\infty$. The finite $t$ case is handled in the Appendix.

\begin{theorem}\label{thm:W:1}
Suppose that Assumptions~\ref{assump:1} and \ref{assump:2} hold. Denote by $\mu,\hat{\mu}$ the unique invariant distributions
of $(\theta_{t})_{t\geq 0}$ and $(\hat{\theta}_{t})_{t\geq 0}$, respectively.
Then, the following inequality holds:
\begin{align}
\mathcal{W}_{1}(\mu,\hat{\mu})
\leq
\left(C_{1}\lambda^{-1}e^{\lambda}+1\right)e^{L}\rho(X_{n},\hat{X}_{n})K_{2}\left(2C_{0}+1\right), \label{eq:wass_bound_invariant}
\end{align}
where $K_1$, $K_2$ and $L$ are defined in Assumption~\ref{assump:1} and Assumption~\ref{assump:2} and $C_0,C_1$ and $\lambda$ are some positive real constants.
\end{theorem}

This theorem shows that, as long as the datasets $X_n$ and $\hat{X}_n$ are close to each other, i.e., $\rho(X_{n},\hat{X}_{n})$ is small, the distance between the solutions of the SDEs \eqref{eq:sde_invar_1} and \eqref{eq:sde_invar_2} will be small as well for any time $t$. This result can be seen as a heavy-tailed version of the results presented in \cite{raginsky2017non,farghly2021time}.

\paragraph{Generalization Bound.} 

By combining Theorem~\ref{thm:W:1} and \eqref{eqn:wass_stab}, we can now easily obtain generalization bound under a Lipschitz surrogate loss function. 


\begin{corollary}
\label{cor:genbound}
Suppose that Assumptions~\ref{assump:1} and \ref{assump:2} hold. Assume that $\ell$ is $\mathcal{L}$-Lipschitz in $\theta$ and $\sup_{x,y \in \mathcal{X}}\|x-y\| \leq D$ for some $D < \infty$. Then the following inequality holds:
\begin{align}
    &\left|\mathbb{E}_{\theta_\infty,X_n}~\left[ \hat{R}(\theta_\infty,X_n) \right] -  R(\theta_\infty)  \right|  \leq\frac{\mathcal{L}D \left(C_{1}\lambda^{-1}e^{\lambda}+1\right)e^{L} K_{2}\left(2C_{0}+1\right)}{n}.
\end{align}
\end{corollary}
The proof of this corollary is straightforward, hence omitted. Similar to Theorem~\ref{thm:W:1}, we presented Corollary~\ref{cor:genbound} for the stationary case, where $t\to\infty$; yet, we shall underline that our theory holds for any finite time $t$. 


Lower bounds on algorithmic stability have been discussed in \citep{Raj2022} for Ornstein-Uhlenbeck process with $\alpha$-stable Levy noise. While comparing with the bound obtained in this work, we can see that the obtained bound has optimal dependence on the number of samples $n$.

Next, we will investigate how the constants in Theorem~\ref{thm:W:1} behave with respect to varying $\alpha$. 

\textbf{Constants in Theorem \ref{thm:W:1}}.
In Theorem~\ref{thm:W:1}, we provided
an upper bound on $\mathcal{W}_{1}(\mu,\hat{\mu})$
which depends on various quantities, 
and our next goal is to figure out how the parameters
$C_{1},\lambda$, $L,K_{2}$ and $C_{0}$ depend on the tail-index $\alpha$.

First, we notice that the parameters $L$ and $K_{2}$
only depend on the loss function.
Second, the parameters $C_{1},\lambda$ come
from the $1$-Wasserstein contraction
in Lemma~\ref{lem:contraction} in the Appendix which 
is a restatement of Proposition 2.2 in \citep{chen2022approximation}; that is, for any $w,y\in\mathbb{R}^{d}$:
\begin{align}
&\mathcal{W}_{1}\left(\mbox{Law}\left(\theta_{t}^{w}\right),\mbox{Law}\left(\theta_{t}^{y}\right)\right)\leq C_{1}e^{-\lambda t}\Vert w-y\Vert,
\\
&\mathcal{W}_{1}\left(\mbox{Law}\left(\hat{\theta}_{t}^{w}\right),\mbox{Law}\left(\hat{\theta}_{t}^{y}\right)\right)\leq C_{1}e^{-\lambda t}\Vert w-y\Vert,
\end{align}
where $\theta_{t}^{w}$ 
to denote the process $\theta_{t}$ that starts
at $\theta_{0}=w$. 
Furthermore, Proposition~2.2 in \citep{chen2022approximation} 
follows from Theorem~1.2. in \citep{Wang2016}.
A careful look at Theorem~1.2. in \citep{Wang2016}
reveals that $C_{1}$, $\lambda$ are independent of the tail-index $\alpha$.

Finally, $C_{0}$ depends on $\alpha$
and it comes from Lemma~\ref{lem:uniform} in the Appendix
which is a restatement of Proposition 2.1. in \citep{chen2022approximation}; that is, 
which says that for any $w\in\mathbb{R}^{d}$:
\begin{align}
&\mathbb{E}\Vert\theta_{t}^{w}\Vert
\leq C_{0}(1+\Vert w\Vert),
\qquad\text{for any $t>0$},
\\
&\mathbb{E}\Vert\hat{\theta}_{t}^{w}\Vert
\leq C_{0}(1+\Vert w\Vert),
\qquad\text{for any $t>0$}.
\end{align}
Notice that Proposition 2.1. in \citep{chen2022approximation}
does not provide an explicit formula for $C_{0}$,
and in the next result, we provide a more refined
estimate to spell out the dependence of $C_{0}$
on the tail-index $\alpha$. 

\begin{lemma}\label{C:0:formula}
Suppose that Assumptions~\ref{assump:1} and \ref{assump:2} hold. For any $w\in\mathbb{R}^{d}$, we have
\begin{align}
&\mathbb{E}\Vert\theta_{t}^{w}\Vert
\leq C_{0}(1+\Vert w\Vert),
\qquad\text{for any $t>0$},
\\
&\mathbb{E}\Vert\hat{\theta}_{t}^{w}\Vert
\leq C_{0}(1+\Vert w\Vert),
\qquad\text{for any $t>0$},
\end{align}
where we can take\small
\begin{align}
&C_{0}:=3+\frac{2\left(K+B\right)}{m}  +\frac{2^{\alpha+1}\Gamma\left(\frac{d+\alpha}{2}\right)\pi^{-d/2}\sigma_{d-1}}{|\Gamma(-\alpha/2)|m}\left(\frac{\sqrt{d} }{ 2-\alpha}
+\frac{1}{\alpha-1}\right),
\end{align}\normalsize
where $\Gamma(\cdot)$ is the gamma function and
$\sigma_{d-1}=2\pi^{\frac{d}{2}}/\Gamma(d/2)$ is the surface area of the unit sphere in $\mathbb{R}^{d}$, 
and $K,B,m$ are defined in Assumption~\ref{assump:2}.
\end{lemma}

Hence, it follows from Theorem~\ref{thm:W:1} and Lemma~\ref{C:0:formula}
that the dependence on the tail-index $\alpha$ is only via
the function:
\begin{equation*}
g(\alpha;d):=\frac{2^{\alpha}\Gamma\left(\frac{d+\alpha}{2}\right)\sqrt{d}}{|\Gamma(-\alpha/2)|(2-\alpha)}
+\frac{2^{\alpha}\Gamma\left(\frac{d+\alpha}{2}\right)}{|\Gamma(-\alpha/2)|(\alpha-1)}.
\end{equation*}
The next result formalizes how the function $g(\alpha;d)$ depends on the tail-index $\alpha$.

\begin{proposition}\label{prop:mono}
Let $\alpha_{0}:=2(c_{0}-1)\in(0,2)$, 
where $c_{0}$ is the unique critical value in $(1,2)$ 
such that the gamma function $\Gamma(x)$ is increasing
for any $x>c_{0}$ and decreasing for any $1<x<c_{0}$.
Then, the following holds.

(i) For any $d\geq d_{0}$, 
where $d_{0}:=\max\left(2,\frac{1}{(\log 2)^{2}(\alpha_{0}-1)^{4}}\right)$,  
the map $\alpha\mapsto g(\alpha;d)$ is increasing in $\alpha\in[\alpha_{0},2]$. 

(ii) For any fixed $d\in\mathbb{N}$, 
the map $\alpha\mapsto g(\alpha;d)$ is decreasing in $\alpha\in[1,\alpha'_{0}]$,
where $\alpha'_{0}\leq\alpha_{0}$ is defined as
$\alpha'_{0}:=\min\left(\alpha_{0},1+\frac{-1+\sqrt{1+4y_{0}^{-1}\sqrt{d}}}{2\sqrt{d}}\right)$,
with $y_{0}:=\log(2)+\frac{1}{2}\psi(d+\frac{\alpha}{2})+\frac{3-\alpha_{0}}{2-\alpha_{0}}$,
where $\psi(\cdot)$ is the digamma function.
\end{proposition}
The proof is given in the Section~\ref{ap:proof_mono} in the Appendix. This result reveals an interesting fact. Depending on the dimension of the parameter vector $d$, the Wasserstein stability bound (Theorem~\ref{thm:W:1}) and the generalization bound (Corollary~\ref{cor:genbound}) exhibit different behaviors with respect to varying $\alpha$. We observe that for sufficiently large $d$, there exists a critical value $\alpha_0$ such that the bound is monotonically increasing for $\alpha \geq \alpha_0$. This suggests that for $d$ large enough, increasing the heaviness of the tails (i.e., smaller $\alpha$) can be beneficial unless $\alpha$ is smaller than $\alpha_0$.

For visualization, we also provide a pictorial illustration of the function $g(\alpha;d)$ in Figure~\ref{fig:fig_1}.  The figure shows the behavior of $g(\alpha;d)$ with respect to $\alpha$ for various dimensions $d$. 
The observed non-globally monotonic behavior for large $d$ indicates that the conclusions of \citep{Raj2022} extend beyond quadratic loss functions. 

On the other hand, \citet{barsbey2021heavy} and \citet{Raj2022} reported several experimental results conducted on neural networks, which illustrated the existence of a non-globally monotonic relation between the generalization error and the heaviness of the tails in practical settings (see Figure~7 in \citep{barsbey2021heavy} and Figure~2 in \citep{Raj2022}).  
Our result brings a stronger theoretical justification to these empirical observations thanks to the generality of our theoretical framework. 

\begin{figure}
    \centering
    \includegraphics[scale=0.55]{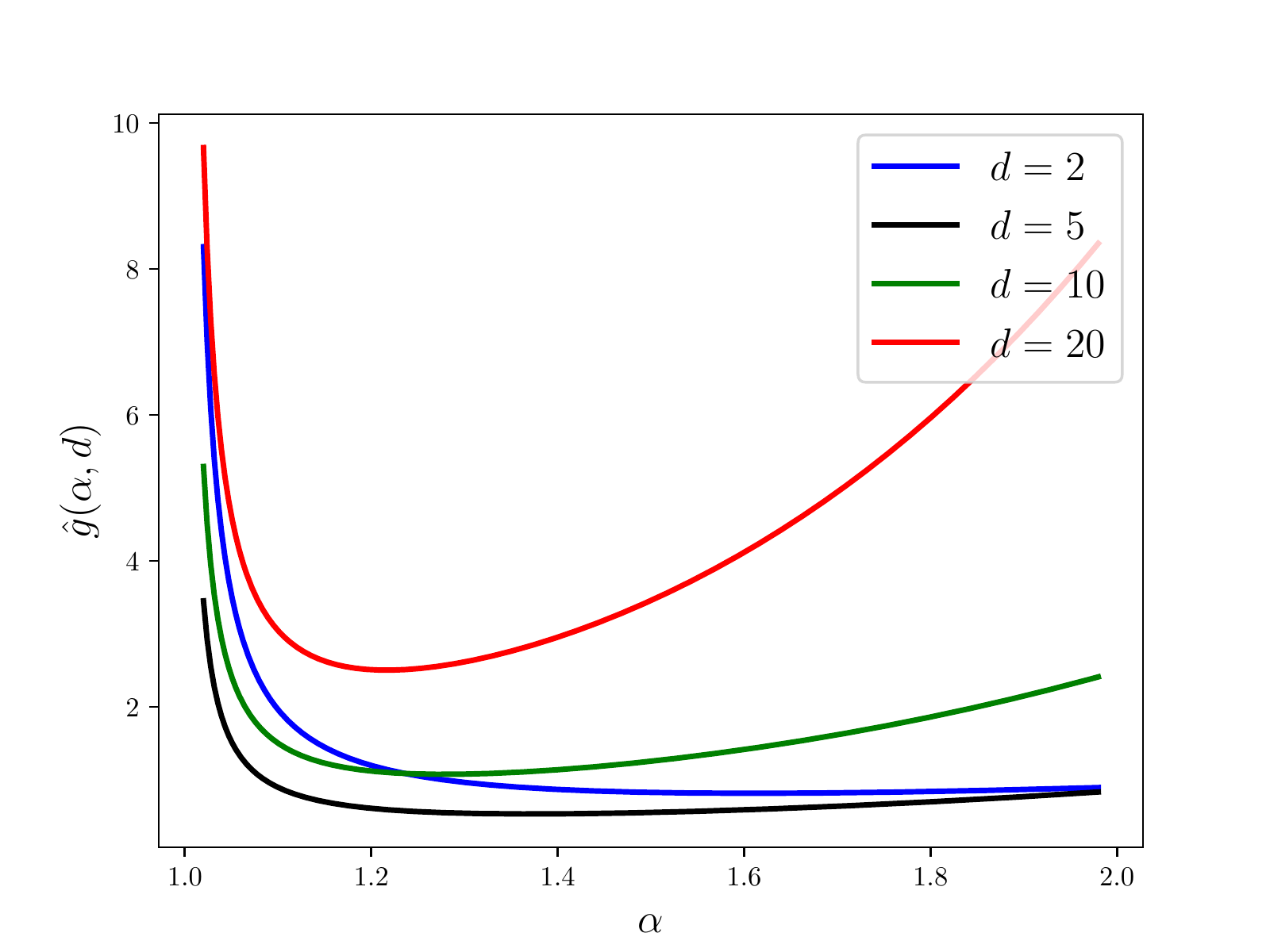}
    \caption{Behavior of $g(\alpha;d)$ with respect to $\alpha$. We scale $g(\alpha;d)$ appropriately to fit all the plots in the same frame which we denote as $\hat{g}(\alpha;d)$. }
    \label{fig:fig_1}
\end{figure}

 \subsection{Infeasibility of $p$-Wasserstein Distance for $p\geq  \alpha$}
Now, that we have provided result for the $1$-Wasserstein distance between the distribution of $\theta$ and $\hat{\theta}$. {A natural question to ask is whether similar results could be obtained more generally in the $p$-Wasserstein distance for some arbitrary $p$. 

Not surprisingly, the following result says that in general we do not expect to control the $p$-Wasserstein distance when $p$ is larger than the tail-index $\alpha$}.

\begin{proposition}
\label{cor:infeas}
Let $d=1$, $\alpha>1$, $\mathcal{X} \subset \mathbb{R}$, and  $f(\theta,x) = (\theta x)^2$. Denote $\mu$ and $\nu$ as the invariant measures of \eqref{eq:sde_invar_1} and \eqref{eq:sde_invar_2}, respectively.
Then for any $p >\alpha$, 
$\mathcal{W}_p(\mu,\nu) = + \infty$.
\end{proposition}
The proof is provided in the Appendix~\ref{ap:infeasibility}.

\subsection{Discrete-Time Dynamics}

Finally, we will illustrate that our theory also extends to the discretizations of the SDEs \eqref{eq:sde_invar_1} and \eqref{eq:sde_invar_2}.
Consider the following Euler-Maruyama discretization:
\begin{equation}
\theta_{k+1}=\theta_{k}-\eta\nabla\hat{F}(\theta_{k},X_{n})+\eta^{1/\alpha}S_{k+1},
\end{equation}
where $S_{k}$ are i.i.d. rotationally invariant alpha-stable random vectors with the characteristic function:
\begin{equation}
\mathbb{E}\left[e^{\imagi\langle u,S_{k}\rangle}\right]=e^{-\Vert u\Vert_{2}^{\alpha}},
\qquad\text{for any $u\in\mathbb{R}^{d}$}.
\end{equation}
Similarly, with input data $\hat{X}_{n}$, we have
\begin{equation}
\hat{\theta}_{k+1}=\hat{\theta}_{k}-\eta\nabla\hat{F}(\hat{\theta}_{k},\hat{X}_{n})+\eta^{1/\alpha}S_{k+1}.
\end{equation}

Let $\mu$ and $\hat{\mu}$ denote the stationary distributions
of continuous-time $\theta_{t}$ and $\hat{\theta}_{t}$ as $t\rightarrow\infty$.
Moreover, let $\nu$ and $\hat{\nu}$ denote the stationary distributions
of discrete-time $\theta_{k}$ and $\hat{\theta}_{k}$ as $k\rightarrow\infty$.
It is proved in \citep{chen2022approximation}
that the $1$-Wasserstein distance of the discretization error
is of order $\eta^{2/\alpha-1}$. 
More precisely, they showed the following result.

\begin{lemma}[Theorem 1.2. in \citet{chen2022approximation}]\label{lem:discretization:error}
Suppose that Assumptions~\ref{assump:1} and \ref{assump:2} hold. Let $m$, $L$ be
as in Assumption~\ref{assump:2}. 
Then, there exists some constant $Q$ (that may depend on $B,m,K,L,M$ from Assumption~\ref{assump:2}) such that
for every $\eta<\min\{1,m/L^{2},1/m\}$, one has
\begin{align}
&\mathcal{W}_{1}(\mu,\nu)
\leq Q\eta^{2/\alpha-1},
\\
&\mathcal{W}_{1}(\hat{\mu},\hat{\nu})
\leq Q\eta^{2/\alpha-1}.
\end{align}
\end{lemma}
By applying the above $1$-Wasserstein bound
for the discretization error in Lemma~\ref{lem:discretization:error} 
and Theorem~\ref{thm:W:1}, 
we obtain the following corollary, 
which provides the $1$-Wasserstein distance 
of the stationary distributions of the discrete-time $(\theta_{k})_{k\geq 0}$
and $(\hat{\theta}_{k})_{k\geq 0}$ processes.

\begin{corollary}\label{cor:discrete}
Under the assumptions in Theorem~\ref{thm:W:1} and Lemma~\ref{lem:discretization:error}, 
we have
\begin{align}
&\mathcal{W}_{1}(\nu,\hat{\nu})
\leq
2Q\eta^{2/\alpha-1} \notag \\
&+\left(C_{1}\lambda^{-1}e^{\lambda}+1\right)e^{L}\rho(X_{n},\hat{X}_{n})K_{2}\left(2C_{0}+1\right).
\end{align}
\end{corollary}

By using the same approach as we used in Corollary~\ref{cor:genbound}, we can easily obtain a generalization bound for the discrete-time as well.
Note that the upper bound in Corollary~\ref{cor:discrete} depends
on the tail-index $\alpha$ only
via $\eta^{2/\alpha-1}$, which 
is increasing in $\alpha$ (since $\eta<1$ as assumed in Lemma~\ref{lem:discretization:error}), and the constant $C_{0}$ which depends on $\alpha$ via
$g(\alpha;d)$ function. 
Therefore, by Proposition~\ref{prop:mono},
the upper bound in Corollary~\ref{cor:discrete} is
increasing in $\alpha\in[\alpha_{0},2]$, 
for any $d\geq d_{0}$, 
where $d_{0}$ and
$\alpha_{0}$ are given in Proposition~\ref{prop:mono}. 
Moreover, the proof of Proposition~\ref{prop:mono}
reveals that $\frac{\partial}{\partial\alpha}g(\alpha;d)$
tends to $-\infty$ as $\alpha$ tends to $1$
and thus there exists some $\alpha''_{0}<\alpha'_{0}$,
where $\alpha'_{0}$ is defined in Proposition~\ref{prop:mono},
such that for any fixed $d\in\mathbb{N}$, the upper bound in Corollary~\ref{cor:discrete} is
decreasing in $\alpha\in[1,\alpha''_{0}]$.
Hence, the conclusions that we obtained for the continuous-time processes remain valid for the discretizations as well.

\vspace{-1mm}
\section{Conclusion}
\vspace{-1mm}

In this work, we studied the relation between the generalization behavior and the heavy tails arising in the SGD dynamics. Previous work on the topic obtained monotonic relationship under strong topological and statistical regularity assumptions, with the exception of the approach in \citep{Raj2022} which was limited to only quadratic losses. 
Following the literature, we considered heavy-tailed SDEs and their discretization for modeling the heavy-tailed behavior of SGD, and showed that the relation is non-monotonic for a general class of losses satisfying a dissipativity condition which generalizes the results of \citep{Raj2022} beyond quadratic losses. 
Our proof technique is based on a novel $1$-Wasserstein stability bound for the symmetric $\alpha$-stable L\'{e}vy-driven SDEs, that model the SGD dynamics. 
Furthermore, our results, when combined with the results of  \citep{raginsky2016information}, yield directly a generalization bound for the class of Lipschitz functions.

\textbf{Future Directions:} As a future research direction, we would like to obtain similar stability bounds without making the dissipativity assumption on the objective function as being done for Langevin Monte Carlo in \citep{kinoshita2022improved}. We would also like to consider specific class of functions (e.g. one-layer neural network) and study the effect of tail-index with other parameters and its effect on the generalization. 
\section*{Acknowledgment}
A.R is supported by the a Marie Sklodowska-Curie
Fellowship (project NN-OVEROPT 101030817). M.G.'s research is supported in part by the grants Office of Naval Research Award Number N00014-21-1-2244, National Science Foundation (NSF) CCF-1814888, NSF DMS-2053485, NSF DMS-1723085. U.\c{S}.'s research is supported by the French government under management of Agence Nationale de la Recherche as part of the ``Investissements d'avenir'' program, reference ANR-19-P3IA-0001 (PRAIRIE 3IA Institute) and the European Research Council Starting Grant DYNASTY – 101039676.
L.Z. is grateful to the support from a Simons Foundation Collaboration Grant and the grants NSF DMS-2053454, NSF DMS-2208303 from the National Science Foundation.




\bibliography{opt-ml,levy,heavy}
\bibliographystyle{icml2023}

\newpage
\appendix
\onecolumn
\begin{center}
\Large \bf Algorithmic stability of heavy-tailed SGD with general loss functions \vspace{3pt}\\ {\normalsize Supplementary Document}
\end{center}
\section{Background on Markov Semigroups}

In this section, we introduce the concept
of Markov semigroups, that will be used
in the proofs of main results in Section~\ref{sec:proofs}.

For a continuous-time Markov process $(X_{t}^{w})_{t\geq 0}$
that starts at $X_{0}=w$, 
its Markov semigroup $P_{t}$ is defined as
for any bounded measurable function $f:\mathbb{R}^{d}\rightarrow\mathbb{R}$,
\begin{equation}
P_{t}f(w)=\mathbb{E}f(X_{t}^{w}),\qquad t\geq 0.
\end{equation}
Similarly, for a discrete-time Markov process $(Y_{k}^{w})_{k=0}^{\infty}$
that starts at $Y_{0}=w$, 
its Markov semigroup $Q_{k}$ is defined as
for any bounded measurable function $f:\mathbb{R}^{d}\rightarrow\mathbb{R}$,
\begin{equation}
Q_{k}f(w)=\mathbb{E}f(Y_{k}^{w}),\qquad k=0,1,2,\ldots.
\end{equation}

\section{Proofs of Main Results}\label{sec:proofs}

In this section, we provide the proofs of main results in our paper.

\subsection{Proof of Theorem~\ref{thm:W:1}} \label{ap:proof_theorem_W_1}
We first provide the theorem statement with all the details. 

\begin{theorem}[Restatement of Theorem~{\ref{thm:W:1}}]\label{thm:W:1:re}
Assume $\theta_{0}=\hat{\theta}_{0}=w$.
Denote by $\mu,\hat{\mu}$ the unique invariant distributions
of $(\theta_{t})_{t\geq 0}$ and $(\hat{\theta}_{t})_{t\geq 0}$ respectively.
The following two statements hold:

(i) For every $N\geq 1$ and $\eta\in(0,1)$, we have the following statements:

(I) If $N=1$, then
\begin{align}
&\mathcal{W}_{1}\left(\mbox{Law}(\theta_{\eta N}),\mbox{Law}(\hat{\theta}_{\eta N})\right)\nonumber
\\
&\leq
\left(K_{1}+\rho(X_{n},\hat{X}_{n})K_{2}\right)(2C)\cdot\left[(1+\Vert w\Vert)\eta^{1+\frac{1}{\alpha}}
+\rho(X_{n},\hat{X}_{n})K_{2}(2\Vert w\Vert+1)\eta\right].
\end{align}

(II) If $2\leq N\leq\eta^{-1}+1$, then 
\begin{align}
&\mathcal{W}_{1}\left(\mbox{Law}(\theta_{\eta N}),\mbox{Law}(\hat{\theta}_{\eta N})\right)
\nonumber
\\
&\leq
\left(K_{1}+\rho(X_{n},\hat{X}_{n})K_{2}\right)(2C)(1+C_{0}(1+\Vert w\Vert))\eta^{1+\frac{1}{\alpha}}
+\rho(X_{n},\hat{X}_{n})K_{2}(2C_{0}(1+\Vert w\Vert)+1)\eta
\nonumber
\\
&\qquad
+e^{L}\left(K_{1}+\rho(X_{n},\hat{X}_{n})K_{2}\right)\left(2C\right)(1+C_{0}(1+\Vert w\Vert))\eta^{\frac{1}{\alpha}}
\nonumber
\\
&\qquad\qquad
+e^{L}\rho(X_{n},\hat{X}_{n})K_{2}(2C_{0}(1+\Vert w\Vert)+1).
\end{align} 

(III) If $N>\eta^{-1}+1$, then
\begin{align}
&\mathcal{W}_{1}\left(\mbox{Law}(\theta_{\eta N}),\mbox{Law}(\hat{\theta}_{\eta N})\right)
\nonumber
\\
&\leq
\left(K_{1}+\rho(X_{n},\hat{X}_{n})K_{2}\right)\left(2C\right)(1+C_{0}(1+\Vert w\Vert))\eta^{1+\frac{1}{\alpha}} +\rho(X_{n},\hat{X}_{n})K_{2}(2C_{0}(1+\Vert w\Vert)+1)\eta
\nonumber
\\
& \qquad \qquad
+\left(C_{1}\lambda^{-1}e^{\lambda}+1\right)
e^{L}\left(K_{1}+\rho(X_{n},\hat{X}_{n})K_{2}\right)\left(2C\right)(1+C_{0}(1+\Vert w\Vert))\eta^{\frac{1}{\alpha}}
\nonumber
\\
&\qquad \qquad \qquad 
+\left(C_{1}\lambda^{-1}e^{\lambda}+1\right)e^{L}\rho(X_{n},\hat{X}_{n})K_{2}(2C_{0}(1+\Vert w\Vert)+1).
\end{align} 
(ii) We have  
\begin{align}
\mathcal{W}_{1}(\mu,\hat{\mu})
\leq
\left(C_{1}\lambda^{-1}e^{\lambda}+1\right)e^{L}\rho(X_{n},\hat{X}_{n})K_{2}\left(2C_{0}+1\right).
\end{align}
\end{theorem}

\begin{proof}[Proof of Theorem~\ref{thm:W:1}]
(i) We first prove part (i).
For any $h\in\text{Lip}(1)$, by the semigroup property, we have
\begin{align*}
P_{N\eta}h(w)-\hat{P}_{N\eta}h(w)
&=\sum_{i=1}^{N}\left(\hat{P}_{(i-1)\eta}P_{(N-i+1)\eta}h(w)-\hat{P}_{i\eta}P_{(N-i)\eta}h(w)\right)
\\
&=\sum_{i=1}^{N}\hat{P}_{(i-1)\eta}(P_{\eta}-\hat{P}_{\eta})P_{(N-i)\eta}h(w).
\end{align*}
Therefore, we can compute that
\begin{align}
&\mathcal{W}_{1}\left(\mbox{Law}(\theta_{\eta N}),\mbox{Law}(\hat{\theta}_{\eta N})\right)
\nonumber
\\
&=\sup_{h\in\text{Lip}(1)}\left|P_{N\eta}h(w)-\hat{P}_{N\eta}h(w)\right|
\nonumber
\\
&\leq
\sup_{h\in\text{Lip}(1)}\left|\hat{P}_{(N-1)\eta}(P_{\eta}-\hat{P}_{\eta})h(w)\right|
+\sum_{i=1}^{N-1}\sup_{h\in\text{Lip}(1)}\left|\hat{P}_{(i-1)\eta}(P_{\eta}-\hat{P}_{\eta})P_{(N-i)\eta}h(w)\right|.\label{two:terms}
\end{align}

Let us first bound the first term in \eqref{two:terms}. 
For any $h\in\text{Lip}(1)$ and $\eta<1$, by applying Lemma~\ref{lem:key}, we get
\begin{align*}
\left|(P_{\eta}-\hat{P}_{\eta})h(w)\right|
\leq
\left(K_{1}+\rho(X_{n},\hat{X}_{n})K_{2}\right)\left(2C\right)(1+\Vert w\Vert)\eta^{1+\frac{1}{\alpha}}
+\rho(X_{n},\hat{X}_{n})K_{2}(2\Vert w\Vert+1)\eta.
\end{align*}
Hence, we have
\begin{align}
&\sup_{h\in\text{Lip}(1)}\left|\hat{P}_{(N-1)\eta}(P_{\eta}-\hat{P}_{\eta})h(w)\right|
\nonumber
\\
&\leq
\left(K_{1}+\rho(X_{n},\hat{X}_{n})K_{2}\right)\left(2C\right)(1+\mathbb{E}\Vert\hat{\theta}_{(N-1)\eta}^{w}\Vert)\eta^{1+\frac{1}{\alpha}}
+\rho(X_{n},\hat{X}_{n})K_{2}(2\mathbb{E}\Vert\hat{\theta}_{(N-1)\eta}^{w}\Vert+1)\eta
\label{N:1}
\\
&\leq
\left(K_{1}+\rho(X_{n},\hat{X}_{n})K_{2}\right)\left(2C\right)(1+C_{0}(1+\Vert w\Vert))\eta^{1+\frac{1}{\alpha}}
+\rho(X_{n},\hat{X}_{n})K_{2}(2C_{0}(1+\Vert w\Vert)+1)\eta,\nonumber
\end{align}
where we applied Lemma~\ref{lem:uniform} to obtain the last inequality above.

Next, let us bound the second term in \eqref{two:terms} and
hence bound the 1-Wasserstein distance $\mathcal{W}_{1}\left(\mbox{Law}(\theta_{\eta N}),\mbox{Law}(\hat{\theta}_{\eta N})\right)$.

We consider three cases: (I) $N=1$; (II) $2\leq N\leq\eta^{-1}+1$
and (III) $N>\eta^{-1}+1$.

Case (I): $N=1$. One can apply \eqref{N:1} and obtain 
\begin{align}
&\mathcal{W}_{1}\left(\mbox{Law}(\theta_{\eta}),\mbox{Law}(\hat{\theta}_{\eta})\right)
\nonumber
\\
&\leq
\left(K_{1}+\rho(X_{n},\hat{X}_{n})K_{2}\right)\left(2C\right)(1+\mathbb{E}\Vert\hat{\theta}_{0}^{w}\Vert)\eta^{1+\frac{1}{\alpha}}
+\rho(X_{n},\hat{X}_{n})K_{2}(2\mathbb{E}\Vert\hat{\theta}_{0}^{w}\Vert+1)\eta
\nonumber
\\
&=\left(K_{1}+\rho(X_{n},\hat{X}_{n})K_{2}\right)\left(2C\right)(1+\Vert w\Vert)\eta^{1+\frac{1}{\alpha}}
+\rho(X_{n},\hat{X}_{n})K_{2}(2\Vert w\Vert+1)\eta.
\end{align}
This completes the proof of part (I).

Case (II): $2\leq N\leq\eta^{-1}+1$.
By Lemma~\ref{lem:key}, for any $i\geq 1$, we have
\begin{align}
&\left|(P_{\eta}-\hat{P}_{\eta})P_{(N-i)\eta}h(w)\right|
\nonumber
\\
&\leq
\Vert\nabla P_{(N-i)\eta}h\Vert_{\infty}
\left[\left(K_{1}+\rho(X_{n},\hat{X}_{n})K_{2}\right)\left(2C\right)(1+\Vert w\Vert)\eta^{1+\frac{1}{\alpha}}
+\rho(X_{n},\hat{X}_{n})K_{2}(2\Vert w\Vert+1)\eta\right]
\nonumber
\\
&\leq
e^{L}\left[\left(K_{1}+\rho(X_{n},\hat{X}_{n})K_{2}\right)\left(2C\right)(1+\Vert w\Vert)\eta^{1+\frac{1}{\alpha}}
+\rho(X_{n},\hat{X}_{n})K_{2}(2\Vert w\Vert+1)\eta\right],\label{ineq:middle}
\end{align}
where we used Lemma~\ref{lem:grad} 
and the fact that for any $i\geq 1$ and $2\leq N\leq\eta^{-1}+1$
we have $(N-i)\eta\leq 1$
in the inequality \eqref{ineq:middle}.

By applying Lemma~\ref{lem:uniform}, we obtain
\begin{align}
&\sup_{h\in\text{Lip}(1)}\left|\tilde{P}_{(i-1)\eta}(P_{\eta}-\hat{P}_{\eta})P_{(N-i)\eta}h(w)\right|
\nonumber
\\
&\leq
e^{L}\left[\left(K_{1}+\rho(X_{n},\hat{X}_{n})K_{2}\right)\left(2C\right)\left(1+\mathbb{E}\Vert\hat{\theta}_{(i-1)\eta}\Vert\right)\eta^{1+\frac{1}{\alpha}}
+\rho(X_{n},\hat{X}_{n})K_{2}\left(2\mathbb{E}\Vert\hat{\theta}_{(i-1)\eta}\Vert+1\right)\eta\right]
\nonumber
\\
&\leq
e^{L}\left(K_{1}+\rho(X_{n},\hat{X}_{n})K_{2}\right)\left(2C\right)\left(1+C_{0}(1+\Vert w\Vert)\right)\eta^{1+\frac{1}{\alpha}}
\nonumber
\\
&\qquad\qquad\qquad\qquad\qquad\qquad\qquad
+e^{L}\rho(X_{n},\hat{X}_{n})K_{2}(2C_{0}(1+\Vert w\Vert)+1)\eta.\label{ineq:non:decay}
\end{align}
Hence, we conclude that
\begin{align}
&\mathcal{W}_{1}\left(\mbox{Law}(\theta_{\eta N}),\mbox{Law}(\hat{\theta}_{\eta N})\right)
\nonumber
\\
&\leq
\sup_{h\in\text{Lip}(1)}\left|\hat{P}_{(N-1)\eta}(P_{\eta}-\hat{P}_{\eta})h(w)\right|
+\sum_{i=1}^{N-1}\sup_{h\in\text{Lip}(1)}\left|\hat{P}_{(i-1)\eta}(P_{\eta}-\hat{P}_{\eta})P_{(N-i)\eta}h(w)\right|
\nonumber
\\
&\leq
\left(K_{1}+\rho(X_{n},\hat{X}_{n})K_{2}\right)\left(2C\right)\left(1+C_{0}(1+\Vert w\Vert)\right)\eta^{1+\frac{1}{\alpha}}
+\rho(X_{n},\hat{X}_{n})K_{2}\left(2C_{0}(1+\Vert w\Vert)+1\right)\eta
\nonumber
\\
&\qquad
+(N-1)e^{L}\left(K_{1}+\rho(X_{n},\hat{X}_{n})K_{2}\right)\left(2C\right)\left(1+C_{0}(1+\Vert w\Vert)\right)\eta^{1+\frac{1}{\alpha}}
\nonumber
\\
&\qquad\qquad\qquad\qquad\qquad\qquad\qquad
+(N-1)e^{L}\rho(X_{n},\hat{X}_{n})K_{2}\left(2C_{0}(1+\Vert w\Vert)+1\right)\eta
\nonumber
\\
&\leq
\left(K_{1}+\rho(X_{n},\hat{X}_{n})K_{2}\right)\left(2C\right)\left(1+C_{0}(1+\Vert w\Vert)\right)\eta^{1+\frac{1}{\alpha}}
+\rho(X_{n},\hat{X}_{n})K_{2}\left(2C_{0}(1+\Vert w\Vert)+1\right)\eta
\nonumber
\\
&\qquad
+e^{L}\left(K_{1}+\rho(X_{n},\hat{X}_{n})K_{2}\right)\left(2C\right)\left(1+C_{0}(1+\Vert w\Vert)\right)\eta^{\frac{1}{\alpha}}
\nonumber
\\
&\qquad\qquad\qquad\qquad\qquad\qquad\qquad
+e^{L}\rho(X_{n},\hat{X}_{n})K_{2}\left(2C_{0}(1+\Vert w\Vert)+1\right).
\end{align}
This completes the proof of (I).

Case (III): $N>\eta^{-1}+1$.
We can compute that
\begin{align*}
&\sup_{h\in\text{Lip}(1)}\left|\hat{P}_{(i-1)\eta}(P_{\eta}-\hat{P}_{\eta})P_{(N-i)\eta}h(w)\right|
\\
&=\sup_{h\in\text{Lip}(1)}\left|\hat{P}_{(i-1)\eta}(P_{\eta}-\hat{P}_{\eta})P_{1}P_{(N-i)\eta-1}h(w)\right|
\\
&\leq\sup_{g\in\text{Lip}(1)}\left|\hat{P}_{(i-1)\eta}(P_{\eta}-\hat{P}_{\eta})P_{1}g(w)\right|
\sup_{h\in\text{Lip}(1)}\Vert\nabla P_{(N-i)\eta-1}h\Vert_{\infty}.
\end{align*}
By Lemma~\ref{lem:contraction}, for any $h\in\text{Lip}(1)$, 
\begin{equation}
\left|P_{t}h(w)-P_{t}h(y)\right|
\leq C_{1}e^{-\lambda t}\Vert w-y\Vert,
\end{equation}
for any $t\geq 0$ and $w,y\in\mathbb{R}^{d}$.
This implies that for any $h\in\text{Lip}(1)$, 
\begin{equation}
\Vert\nabla P_{t}h\Vert_{\infty}
\leq C_{1}e^{-\lambda t}.
\end{equation}
Hence, we conclude that
\begin{align*}
\sup_{h\in\text{Lip}(1)}\left|\hat{P}_{(i-1)\eta}(P_{\eta}-\hat{P}_{\eta})P_{(N-i)\eta}h(w)\right|
\leq
C_{1}e^{-\lambda((N-i)\eta-1)}\sup_{g\in\text{Lip}(1)}\left|\hat{P}_{(i-1)\eta}(P_{\eta}-\hat{P}_{\eta})P_{1}g(w)\right|,
\end{align*}
where $i\leq\lfloor N-\eta^{-1}\rfloor$.

Moreover, by Lemma~\ref{lem:key}, we have
\begin{align}
&\left|P_{\eta}P_{1}g(w)-\hat{P}_{\eta}P_{1}g(w)\right|
\nonumber
\\
&\leq
\Vert\nabla P_{1}g\Vert_{\infty}
\left[\left(K_{1}+\rho(X_{n},\hat{X}_{n})K_{2}\right)\left(2C\right)(1+\Vert w\Vert)\eta^{1+\frac{1}{\alpha}}
+\rho(X_{n},\hat{X}_{n})K_{2}(2\Vert w\Vert+1)\eta\right]
\nonumber
\\
&\leq
e^{L}\left[\left(K_{1}+\rho(X_{n},\hat{X}_{n})K_{2}\right)\left(2C\right)(1+\Vert w\Vert)\eta^{1+\frac{1}{\alpha}}
+\rho(X_{n},\hat{X}_{n})K_{2}(2\Vert w\Vert+1)\eta\right],
\end{align}
which by Lemma~\ref{lem:uniform} implies that
\begin{align*}
&\sup_{g\in\text{Lip}(1)}\left|\hat{P}_{(i-1)\eta}(P_{\eta}-\hat{P}_{\eta})P_{1}g(w)\right|
\\
&\leq
e^{L}\left[\left(K_{1}+\rho(X_{n},\hat{X}_{n})K_{2}\right)\left(2C\right)(1+\mathbb{E}\Vert\theta_{(i-1)\eta}^{w}\Vert)\eta^{1+\frac{1}{\alpha}}
+\rho(X_{n},\hat{X}_{n})K_{2}\left(2\mathbb{E}\Vert\theta_{(i-1)\eta}^{w}\Vert+1\right)\eta\right]
\\
&\leq
e^{L}\left[\left(K_{1}+\rho(X_{n},\hat{X}_{n})K_{2}\right)\left(2C\right)\left(1+C_{0}(1+\Vert w\Vert)\right)\eta^{1+\frac{1}{\alpha}}
+\rho(X_{n},\hat{X}_{n})K_{2}\left(2C_{0}(1+\Vert w\Vert)+1\right)\eta\right].
\end{align*}
Therefore, we have
\begin{align*}
&\sum_{i=1}^{\lfloor N-\eta^{-1}\rfloor}\sup_{h\in\text{Lip}(1)}\left|\hat{P}_{(i-1)\eta}(P_{\eta}-\hat{P}_{\eta})P_{(N-i)\eta}h(w)\right|
\\
&\leq
\sum_{i=1}^{\lfloor N-\eta^{-1}\rfloor}C_{1}e^{-\lambda((N-i)\eta-1)}
e^{L}\left(K_{1}+\rho(X_{n},\hat{X}_{n})K_{2}\right)\left(2C\right)\left(1+C_{0}(1+\Vert w\Vert)\right)\eta^{1+\frac{1}{\alpha}}
\\
&\qquad
+\sum_{i=1}^{\lfloor N-\eta^{-1}\rfloor}C_{1}e^{-\lambda((N-i)\eta-1)}
e^{L}\rho(X_{n},\hat{X}_{n})K_{2}\left(2C_{0}(1+\Vert w\Vert)+1\right)\eta
\\
&\leq
C_{1}\lambda^{-1}e^{\lambda}
e^{L}\left(K_{1}+\rho(X_{n},\hat{X}_{n})K_{2}\right)\left(2C\right)(1+C_{0}(1+\Vert w\Vert))\eta^{\frac{1}{\alpha}}
\\
&\qquad
+C_{1}\lambda^{-1}e^{\lambda}e^{L}\rho(X_{n},\hat{X}_{n})K_{2}\left(2C_{0}(1+\Vert w\Vert)+1\right),
\end{align*}
where we used the fact that
\begin{align*}
\sum_{i=1}^{\lfloor N-\eta^{-1}\rfloor}e^{-\lambda((N-i)\eta-1)}
\leq
e^{\lambda}\int_{\lfloor\eta^{-1}\rfloor-1}^{N-1}e^{-\lambda\eta r}dr
\leq
e^{\lambda}\eta^{-1}\int_{0}^{\infty}e^{-\lambda r}dr
=\lambda e^{\lambda}\eta^{-1}.
\end{align*}

Next, when $i\geq\lfloor N-\eta^{-1}\rfloor+1$, 
by applying \eqref{ineq:non:decay}, we have
\begin{align*}
&\sum_{i=\lfloor N-\eta^{-1}\rfloor+1}^{N-1}
\sup_{h\in\text{Lip}(1)}\left|\hat{P}_{(i-1)\eta}(P_{\eta}-\hat{P}_{\eta})P_{(N-i)\eta}h(w)\right|
\\
&\leq
\sum_{i=\lfloor N-\eta^{-1}\rfloor+1}^{N-1}
e^{L}\left(K_{1}+\rho(X_{n},\hat{X}_{n})K_{2}\right)\left(2C\right)\left(1+C_{0}(1+\Vert w\Vert)\right)\eta^{1+\frac{1}{\alpha}}
\nonumber
\\
&\qquad\qquad\qquad\qquad\qquad\qquad\qquad
+\sum_{i=\lfloor N-\eta^{-1}\rfloor+1}^{N-1}e^{L}\rho(X_{n},\hat{X}_{n})K_{2}\left(2C_{0}(1+\Vert w\Vert)+1\right)\eta
\\
&\leq
e^{L}\left(K_{1}+\rho(X_{n},\hat{X}_{n})K_{2}\right)\left(2C\right)\left(1+C_{0}(1+\Vert w\Vert)\right)\eta^{\frac{1}{\alpha}}
\nonumber
\\
&\qquad\qquad\qquad\qquad\qquad\qquad\qquad
+\sum_{i=\lfloor N-\eta^{-1}\rfloor+1}^{N-1}e^{L}\rho(X_{n},\hat{X}_{n})K_{2}\left(2C_{0}(1+\Vert w\Vert)+1\right).
\end{align*}
Therefore, we obtain
\begin{align*}
&\sum_{i=1}^{N-1}\sup_{h\in\text{Lip}(1)}\left|\hat{P}_{(i-1)\eta}(P_{\eta}-\hat{P}_{\eta})P_{(N-i)\eta}h(w)\right|
\\
&\leq
\left(C_{1}\lambda^{-1}e^{\lambda}+1\right)
e^{L}\left(K_{1}+\rho(X_{n},\hat{X}_{n})K_{2}\right)\left(2C\right)\left(1+C_{0}(1+\Vert w\Vert)\right)\eta^{\frac{1}{\alpha}}
\\
&\qquad
+\left(C_{1}\lambda^{-1}e^{\lambda}+1\right)e^{L}\rho(X_{n},\hat{X}_{n})K_{2}\left(2C_{0}(1+\Vert w\Vert)+1\right).
\end{align*}

Hence, we conclude that
\begin{align}
&\mathcal{W}_{1}\left(\mbox{Law}(\theta_{\eta N}),\mbox{Law}(\hat{\theta}_{\eta N})\right)
\nonumber
\\
&\leq
\sup_{h\in\text{Lip}(1)}\left|\hat{P}_{(N-1)\eta}\left(P_{\eta}-\hat{P}_{\eta}\right)h(w)\right|
+\sum_{i=1}^{N-1}\sup_{h\in\text{Lip}(1)}\left|\hat{P}_{(i-1)\eta}\left(P_{\eta}-\hat{P}_{\eta}\right)P_{(N-i)\eta}h(w)\right|
\nonumber
\\
&\leq
\left(K_{1}+\rho(X_{n},\hat{X}_{n})K_{2}\right)\left(2C\right)\left(1+C_{0}(1+\Vert w\Vert)\right)\eta^{1+\frac{1}{\alpha}}
+\rho(X_{n},\hat{X}_{n})K_{2}\left(2C_{0}(1+\Vert w\Vert)+1\right)\eta
\nonumber
\\
&\qquad
+\left(C_{1}\lambda^{-1}e^{\lambda}+1\right)
e^{L}\left(K_{1}+\rho(X_{n},\hat{X}_{n})K_{2}\right)\left(2C\right)\left(1+C_{0}(1+\Vert w\Vert)\right)\eta^{\frac{1}{\alpha}}
\nonumber
\\
&\qquad\qquad
+\left(C_{1}\lambda^{-1}e^{\lambda}+1\right)e^{L}\rho(X_{n},\hat{X}_{n})K_{2}\left(2C_{0}(1+\Vert w\Vert)+1\right).
\end{align}
This completes the proof of part (III).

(ii) Now, we are ready prove part (ii).
By triangle inequality, 
\begin{align*}
\mathcal{W}_{1}\left(\mu,\hat{\mu}\right)
\leq
\mathcal{W}_{1}\left(\mbox{Law}(\theta_{\eta N}),\mu\right)
+\mathcal{W}_{1}\left(\mbox{Law}(\theta_{\eta N}),\mbox{Law}(\hat{\theta}_{\eta N})\right)
+\mathcal{W}_{1}\left(\mbox{Law}(\hat{\theta}_{\eta N}),\hat{\mu}\right).
\end{align*}
It follows from Lemma~\ref{lem:uniform} that
by letting $N\rightarrow\infty$, we have
\begin{align}
&\mathcal{W}_{1}\left(\mu,\hat{\mu}\right)
\nonumber
\\
&\leq
\limsup_{N\rightarrow\infty}\mathcal{W}_{1}\left(\mbox{Law}(\theta_{\eta N}),\mbox{Law}(\hat{\theta}_{\eta N})\right)
\nonumber
\\
&\leq
\left(K_{1}+\rho(X_{n},\hat{X}_{n})K_{2}\right)\left(2C\right)\left(1+C_{0}(1+\Vert w\Vert)\right)\eta^{1+\frac{1}{\alpha}}
+\rho(X_{n},\hat{X}_{n})K_{2}\left(2C_{0}(1+\Vert w\Vert)+1\right)\eta
\nonumber
\\
&\qquad
+\left(C_{1}\lambda^{-1}e^{\lambda}+1\right)
e^{L}\left(K_{1}+\rho(X_{n},\hat{X}_{n})K_{2}\right)\left(2C\right)\left(1+C_{0}(1+\Vert w\Vert)\right)\eta^{\frac{1}{\alpha}}
\nonumber
\\
&\qquad\qquad
+\left(C_{1}\lambda^{-1}e^{\lambda}+1\right)e^{L}\rho(X_{n},\hat{X}_{n})K_{2}\left(2C_{0}(1+\Vert w\Vert)+1\right),\label{let:eta:x:0}
\end{align}
where we used part (III) from part (i).
Since $\mathcal{W}_{1}\left(\mu,\hat{\mu}\right)$
is independent of $\eta$ and the initial state $x\in\mathbb{R}^{d}$, 
we can set $\eta=0$ and $x=0$ in \eqref{let:eta:x:0} and conclude that
\begin{align*}
\mathcal{W}_{1}\left(\mu,\hat{\mu}\right)
\leq
\left(C_{1}\lambda^{-1}e^{\lambda}+1\right)e^{L}\rho(X_{n},\hat{X}_{n})K_{2}\left(2C_{0}+1\right).
\end{align*}
The proof is complete.
\end{proof}

\subsection{Proof of Lemma~\ref{C:0:formula}} \label{ap:proof_co}

\begin{proof}[Proof of Lemma~\ref{C:0:formula}]
First of all, the infinitesimal generator of $\theta_{t}$ process is given by
\begin{equation}
\mathcal{L}^{\alpha}f(\theta)
=\left\langle-\nabla\hat{F}(\theta,X_{n}),\nabla f(\theta)\right\rangle
+(-\Delta)^{\alpha/2}f(\theta),
\end{equation}
where $(-\Delta)^{\alpha/2}$ is the fractional Laplacian operator
defined as a principal value integral:
\begin{equation}
(-\Delta)^{\alpha/2}f(\theta)
=d_{\alpha}\cdot\text{p.v.}\int_{\mathbb{R}^{d}}(f(\theta+y)-f(\theta))\frac{dy}{\Vert y\Vert^{\alpha+d}},
\end{equation}
where (see e.g. \citep{Wang2016})
\begin{equation}
d_{\alpha}:=\frac{2^{\alpha}\Gamma\left(\frac{d+\alpha}{2}\right)\pi^{-d/2}}{|\Gamma(-\alpha/2)|}.
\end{equation}

Next, let $V(w):=(1+\Vert w\Vert^{2})^{1/2}$. 
We derive from Assumption~\ref{assump:2} that 
for any dataset $X_{n}\in\mathcal{X}^{n}$, 
we have the following property:
\begin{align*}
&\Vert\nabla \hat{F}(0,X_{n})\Vert\leq B,
\\
&\left\langle\nabla \hat{F}(\theta_{1},X_{n})-\nabla \hat{F}(\theta_{2},X_{n}),\theta_{1}-\theta_{2}\right\rangle
\leq
-m\Vert\theta_{1}-\theta_{2}\Vert^{2}+K,
\end{align*}
and
\begin{align*}
\left\Vert\nabla_{v}\nabla \hat{F}(\theta,X_{n})\right\Vert
\leq L\Vert v\Vert,
\qquad
\left\Vert\nabla_{v_{1}}\nabla_{v_{2}}\nabla \hat{F}(\theta,X_{n})\right\Vert
\leq
M\Vert v_{1}\Vert\Vert v_{2}\Vert,
\end{align*}
for any $v,v_1, v_2\in\mathbb{R}^{d}$
so that we can apply Proposition 2.1. in \citep{chen2022approximation}.
It is shown in the proof of Proposition 2.1. in \citep{chen2022approximation}
that $V\in\mathcal{D}(\mathcal{L}^{\alpha})$, i.e.
the domain of the infinitesimal generator $\mathcal{L}^{\alpha}$
and moreover 
\begin{equation}\label{V:ineq}
\mathcal{L}^{\alpha}V(w)\leq-\lambda_{1}V(w)+q_{1},
\end{equation}
where
\begin{equation}
\lambda_{1}:=\frac{1}{2}m,
\qquad
q_{1}:=m+K+B+C_{d,\alpha},
\end{equation}
where
\begin{equation}
C_{d,\alpha}:=\frac{d_{\alpha}\sqrt{d}\sigma_{d-1}}{2-\alpha}+\frac{d_{\alpha}\sigma_{d-1}}{\alpha-1}
=\frac{2^{\alpha}\Gamma\left(\frac{d+\alpha}{2}\right)\pi^{-d/2}\sqrt{d}\sigma_{d-1}}{|\Gamma(-\alpha/2)|(2-\alpha)}
+\frac{2^{\alpha}\Gamma\left(\frac{d+\alpha}{2}\right)\pi^{-d/2}\sigma_{d-1}}{|\Gamma(-\alpha/2)|(\alpha-1)},
\end{equation}
where $\sigma_{d-1}:=2\pi^{\frac{d}{2}}/\Gamma(d/2)$ is the surface area of the unit sphere in $\mathbb{R}^{d}$, a positive constant that depends only on $d$.

Next, let us define the extended infinitesimal generator $\mathcal{L}_{t}^{\alpha}$:
\begin{equation}
\mathcal{L}_{t}^{\alpha}f(t,\theta)
:=\partial_{t}f(t,\theta)+\mathcal{L}^{\alpha}f(t,\theta).
\end{equation}
Then, it follows from \eqref{V:ineq} that
\begin{equation}
\mathcal{L}_{t}^{\alpha}e^{\lambda_{1}t}V(\theta)
=\lambda_{1}e^{\lambda_{1}t}V(\theta)
+e^{\lambda_{1}t}\mathcal{L}^{\alpha}V(\theta)
\leq
\lambda_{1}e^{\lambda_{1}t}V(\theta)
+e^{\lambda_{1}t}\left(-\lambda_{1}V(w)+q_{1}\right)
=q_{1}e^{\lambda_{1}t}.
\end{equation}
By Dynkin's formula, 
\begin{align*}
\mathbb{E}\left[e^{\lambda_{1}t}V\left(\theta_{t}^{w}\right)\right]
&=V(w)+\mathbb{E}\left[\int_{0}^{t}\mathcal{L}_{s}^{\alpha}e^{\lambda_{1}s}V\left(\theta_{s}^{w}\right)ds\right]
\\
&\leq V(w)+\int_{0}^{t}q_{1}e^{\lambda_{1}s}ds
=V(w)+q_{1}\frac{e^{\lambda_{1}t}-1}{\lambda_{1}},
\end{align*}
which implies that
\begin{equation}
\mathbb{E}\Vert\theta_{t}^{w}\Vert
\leq
\mathbb{E}\left[V\left(\theta_{t}^{w}\right)\right]
\leq
e^{-\lambda_{1}t}V(w)+q_{1}\frac{1-e^{-\lambda_{1}t}}{\lambda_{1}}
\leq
1+\Vert w\Vert+\frac{q_{1}}{\lambda_{1}},
\end{equation}
where we used $V(w)=(1+\Vert w\Vert^{2})^{1/2}$ and 
the inequality $\Vert w\Vert\leq(1+\Vert w\Vert^{2})^{1/2}\leq 1+\Vert w\Vert$.
Hence, we have
\begin{equation}
\mathbb{E}\Vert\theta_{t}^{w}\Vert
\leq C_{0}(1+\Vert w\Vert),
\end{equation}
where we take
\begin{align*}
C_{0}&:=1+\frac{q_{1}}{\lambda_{1}}
=1+\frac{2\left(m+K+B+C_{d,\alpha}\right)}{m}
\\
&=3+\frac{2\left(K+B\right)}{m}
+\frac{2}{m}\left(\frac{2^{\alpha}\Gamma\left(\frac{d+\alpha}{2}\right)\pi^{-d/2}\sqrt{d}\sigma_{d-1}}{|\Gamma(-\alpha/2)|(2-\alpha)}
+\frac{2^{\alpha}\Gamma\left(\frac{d+\alpha}{2}\right)\pi^{-d/2}\sigma_{d-1}}{|\Gamma(-\alpha/2)|(\alpha-1)}\right).
\end{align*}
Since $C_{0}$ in the above bound is uniform
in the dataset, similarly, we also have
\begin{equation}
\mathbb{E}\Vert\hat{\theta}_{t}^{w}\Vert
\leq C_{0}(1+\Vert w\Vert),
\end{equation}
which completes the proof. 
\end{proof}

\subsection{Proof of Proposition~\ref{prop:mono}} \label{ap:proof_mono}

\begin{proof}[Proof of Proposition~\ref{prop:mono}]
Let us first prove part (i). 
First, we can re-write $g(\alpha;d)$ as
\begin{equation}\label{recall:g}
g(\alpha;d)=\frac{2^{\alpha}\Gamma\left(\frac{d+\alpha}{2}\right)}{|\Gamma(-\alpha/2)|(2-\alpha)}
\left(\sqrt{d}+\frac{2-\alpha}{\alpha-1}\right).
\end{equation}
By the properties of the gamma function, 
we have
\begin{align*}
\Gamma\left(2-\frac{\alpha}{2}\right)
&=\left(1-\frac{\alpha}{2}\right)\Gamma\left(1-\frac{\alpha}{2}\right)
\\
&=\left(1-\frac{\alpha}{2}\right)\frac{-\alpha}{2}\Gamma(-\alpha/2).
\end{align*}
Therefore, we have
\begin{equation}
|\Gamma(-\alpha/2)|(2-\alpha)=\frac{4}{\alpha}\Gamma\left(2-\frac{\alpha}{2}\right).
\end{equation}
Moreover, by the properties of the gamma function,
\begin{align*}
&\Gamma\left(2-\frac{\alpha}{2}\right)
=\Gamma\left(1-\left(\frac{\alpha}{2}-1\right)\right)
\\
&=\frac{\pi}{\sin\left(\pi\left(\frac{\alpha}{2}-1\right)\right)\Gamma\left(\frac{\alpha}{2}-1\right)}
=\frac{\pi\left(\frac{\alpha}{2}-1\right)}{\sin\left(\pi\left(\frac{\alpha}{2}-1\right)\right)\Gamma\left(\frac{\alpha}{2}\right)}.
\end{align*} 
Hence, we conclude that\small
\begin{align*}
g(\alpha;d)&=2^{\alpha-2}\alpha\Gamma\left(\frac{d+\alpha}{2}\right)\Gamma\left(\frac{\alpha}{2}\right)  \cdot\frac{\sin\left(\pi\left(1-\frac{\alpha}{2}\right)\right)}{\pi\left(1-\frac{\alpha}{2}\right)}
\left(\sqrt{d}+\frac{2-\alpha}{\alpha-1}\right),
\end{align*}\normalsize
where we used $\sin(-x)=-\sin(x)$ for any $x\in\mathbb{R}$. 
Let us define
$h(x):=\frac{\sin(x)}{x}$ for any $0\leq x\leq\pi/2$.
We can compute that $h'(x)=\frac{x\cos(x)-\sin(x)}{x^{2}}$.
Let 
$p(x):=x\cos(x)-\sin(x)$. 
Then $p(0)=0$ and $p'(x)=-x\sin(x)<0$
for any $0<x<\pi/2$ which implies that $p(x)<0$ and thus $h'(x)<0$ for any $0<x<\pi/2$.
Hence $h(x)$ is decreasing in $x$ for any $0\leq x\leq\pi/2$. 
As a result, the map
\begin{align}
&\alpha\mapsto\frac{\sin\left(\pi\left(1-\frac{\alpha}{2}\right)\right)}{\pi\left(1-\frac{\alpha}{2}\right)}
\quad
\text{is increasing in $\alpha$}
 ~\text{for any $1<\alpha<2$.}
\label{map:increasing}
\end{align}

It is well known that gamma function $x\mapsto\Gamma(x)$ is 
log-convex for $x>0$ and thus convex for any $x>0$. 
Since $\Gamma(1)=\Gamma(2)=1$, there exists
a unique critical value $c_{0}\in(1,2)$ such that
the gamma function $x\mapsto\Gamma(x)$ is increasing
for any $x\geq c_{0}$ and decreasing for any $1\leq x\leq c_{0}$. 

Next, for any given $\alpha_{0}\in(1,2)$ such that $1+\frac{\alpha_{0}}{2}\geq c_{0}$, 
we have for any $2\geq\alpha_{2}>\alpha_{1}\geq\alpha_{0}$ and $d\geq 2$,  
\begin{align}
\nonumber \frac{g(\alpha_{2};d)}{g(\alpha_{1};d)} 
&=\frac{2^{\alpha_{2}-2}\alpha_{2}\Gamma\left(\frac{d+\alpha_{2}}{2}\right)\Gamma\left(\frac{\alpha_{2}}{2}\right)
\frac{\sin\left(\pi\left(1-\frac{\alpha_{2}}{2}\right)\right)}{\pi\left(1-\frac{\alpha_{2}}{2}\right)}
\left(\sqrt{d}+\frac{2-\alpha_{2}}{\alpha_{2}-1}\right)}
{2^{\alpha_{1}-2}\alpha_{1}\Gamma\left(\frac{d+\alpha_{1}}{2}\right)\Gamma\left(\frac{\alpha_{1}}{2}\right)
\frac{\sin\left(\pi\left(1-\frac{\alpha_{1}}{2}\right)\right)}{\pi\left(1-\frac{\alpha_{1}}{2}\right)}
\left(\sqrt{d}+\frac{2-\alpha_{1}}{\alpha_{1}-1}\right)}
\nonumber
\\
&=\frac{2^{\alpha_{2}}\Gamma\left(\frac{d+\alpha_{2}}{2}\right)\Gamma\left(1+\frac{\alpha_{2}}{2}\right)
\frac{\sin\left(\pi\left(1-\frac{\alpha_{2}}{2}\right)\right)}{\pi\left(1-\frac{\alpha_{2}}{2}\right)}
\left(\sqrt{d}+\frac{2-\alpha_{2}}{\alpha_{2}-1}\right)}
{2^{\alpha_{1}}\Gamma\left(\frac{d+\alpha_{1}}{2}\right)\Gamma\left(1+\frac{\alpha_{1}}{2}\right)
\frac{\sin\left(\pi\left(1-\frac{\alpha_{1}}{2}\right)\right)}{\pi\left(1-\frac{\alpha_{1}}{2}\right)}
\left(\sqrt{d}+\frac{2-\alpha_{1}}{\alpha_{1}-1}\right)}
\nonumber
\\
&\geq
2^{\alpha_{2}-\alpha_{1}}\frac{\sqrt{d}+\frac{2-\alpha_{2}}{\alpha_{2}-1}}{\sqrt{d}+\frac{2-\alpha_{1}}{\alpha_{1}-1}},
\end{align} 
where we used \eqref{map:increasing} and the fact 
that the gamma function $x\mapsto\Gamma(x)$ is increasing
in $x\geq 1+\frac{\alpha_{0}}{2}\geq c_{0}$. 

Next, let us define the function:
\begin{equation}
q(x):=2^{x-\alpha_{1}}\frac{\sqrt{d}+\frac{2-x}{x-1}}{\sqrt{d}+\frac{2-\alpha_{1}}{\alpha_{1}-1}},
\end{equation}
where $2\geq x\geq\alpha_{1}\geq\alpha_{0}$. 
It is clear that $q(\alpha_{1})=1$ and moreover, we can compute that
\small
\begin{align}
q'(x)&=\log(2)2^{x-\alpha_{1}}\frac{\sqrt{d}+\frac{2-x}{x-1}}{\sqrt{d}+\frac{2-\alpha_{1}}{\alpha_{1}-1}}
-\frac{2^{x-\alpha_{1}}}{(x-1)^{2}}\frac{1}{\sqrt{d}+\frac{2-\alpha_{1}}{\alpha_{1}-1}}
\nonumber
\\
&=\frac{2^{x-\alpha_{1}}}{\sqrt{d}+\frac{2-\alpha_{1}}{\alpha_{1}-1}}
\left(\log(2)\left(\sqrt{d}+\frac{2-x}{x-1}\right)-\frac{1}{(x-1)^{2}}\right)
\nonumber
\\
&\geq
\frac{2^{x-\alpha_{1}}}{\sqrt{d}+\frac{2-\alpha_{1}}{\alpha_{1}-1}}
\left(\log(2)\sqrt{d}-\frac{1}{(\alpha_{0}-1)^{2}}\right)
\geq 0,
\end{align}\normalsize
provided that
\begin{equation}\label{d:holds}
d\geq\frac{1}{(\log 2)^{2}(\alpha_{0}-1)^{4}}.
\end{equation}
This implies that $q(x)$ is increasing for $2\geq x\geq\alpha_{1}\geq\alpha_{0}$
provided that $d\geq 2$, $1+\frac{\alpha_{0}}{2}\geq c_{0}$ and \eqref{d:holds} holds.
Hence, we conclude that $g(\alpha_{2};d)\geq g(\alpha_{1};d)$
for any 
$d\geq d_{0}=\max\left(2,\frac{1}{(\log 2)^{2}(\alpha_{0}-1)^{4}}\right)$,
and $2\geq\alpha_{2}\geq\alpha_{1}\geq\alpha_{0}$.

Now, let us prove part (ii) of Proposition~\ref{prop:mono}. 
We recall from \eqref{recall:g} that
\begin{equation}
g(\alpha;d)=\frac{2^{\alpha}\Gamma\left(\frac{d+\alpha}{2}\right)}{|\Gamma(-\alpha/2)|(2-\alpha)}
\left(\sqrt{d}+\frac{2-\alpha}{\alpha-1}\right).
\end{equation}
We can compute that
\begin{align*}
\frac{\partial}{\partial\alpha}g(\alpha;d)
=\frac{2^{\alpha}\Gamma\left(\frac{d+\alpha}{2}\right)}{|\Gamma(-\alpha/2)|(2-\alpha)}
\frac{-1}{(\alpha-1)^{2}}
+\frac{\partial}{\partial\alpha}\left\{\frac{2^{\alpha}\Gamma\left(\frac{d+\alpha}{2}\right)}{\Gamma(-\alpha/2)(\alpha-2)}\right\}
\left(\sqrt{d}+\frac{2-\alpha}{\alpha-1}\right),
\end{align*}
where we can further compute that
\begin{align*}
\frac{\partial}{\partial\alpha}\left\{\frac{2^{\alpha}\Gamma\left(\frac{d+\alpha}{2}\right)}{\Gamma(-\alpha/2)(\alpha-2)}\right\}
&=\frac{\log(2)2^{\alpha}\Gamma\left(\frac{d+\alpha}{2}\right)+2^{\alpha-1}\Gamma\left(\frac{d+\alpha}{2}\right)\psi\left(\frac{d+\alpha}{2}\right)}{\Gamma(-\alpha/2)(\alpha-2)}
\\
&\qquad\qquad
-\frac{2^{\alpha}\Gamma\left(\frac{d+\alpha}{2}\right)\left(-\frac{1}{2}(\alpha-2)\psi(-\frac{\alpha}{2})
+1\right)}{\Gamma(-\alpha/2)(\alpha-2)^{2}},
\end{align*}
where $\psi(\cdot)$ denotes the digamma function.
This implies that
\begin{align}\label{g:derivative}
\frac{\partial}{\partial\alpha}g(\alpha;d)
=\frac{2^{\alpha}\Gamma\left(\frac{d+\alpha}{2}\right)}{|\Gamma(-\alpha/2)|(2-\alpha)}p(\alpha;d),
\end{align}
where
\begin{align*}
p(\alpha;d)&:=\frac{-1}{(\alpha-1)^{2}}
+\left(\log(2)+\frac{1}{2}\psi\left(\frac{d+\alpha}{2}\right)\right)\left(\sqrt{d}+\frac{2-\alpha}{\alpha-1}\right)
\\
&\qquad
+\left(\frac{1}{2}\psi\left(-\frac{\alpha}{2}\right)+\frac{1}{2-\alpha}\right)\left(\sqrt{d}+\frac{2-\alpha}{\alpha-1}\right).
\end{align*}
By the property of the digamma function, 
we have $\psi(-\frac{\alpha}{2})=\psi(1-\frac{\alpha}{2})+\frac{2}{\alpha}$
and $\psi(x)$ is increasing in $x>0$ and $\psi(-1/2)<0$. 
Therefore, for any $1<\alpha\leq\alpha_{0}$, we have
\begin{align*}
p(\alpha;d)&=\frac{-1}{(\alpha-1)^{2}}
+\left(\log(2)+\frac{1}{2}\psi\left(\frac{d+\alpha}{2}\right)\right)\left(\sqrt{d}+\frac{2-\alpha}{\alpha-1}\right)
\\
&\qquad
+\left(\frac{1}{2}\psi\left(1-\frac{\alpha}{2}\right)+\frac{1}{\alpha}+\frac{1}{2-\alpha}\right)\left(\sqrt{d}+\frac{2-\alpha}{\alpha-1}\right)
\\
&\leq
\frac{-1}{(\alpha-1)^{2}}
+\left(\log(2)+\frac{1}{2}\psi\left(\frac{d+\alpha_{0}}{2}\right)\right)\left(\sqrt{d}+\frac{1}{\alpha-1}\right)
\\
&\qquad\qquad\qquad\qquad
+\left(1+\frac{1}{2-\alpha_{0}}\right)\left(\sqrt{d}+\frac{1}{\alpha-1}\right).
\end{align*}
It follows that $p(\alpha;d)\leq 0$ holds if
\begin{equation}\label{y:eqn}
y_{0}\sqrt{d}(\alpha-1)^{2}+y_{0}(\alpha-1)-1\leq 0,
\end{equation}
where $y_{0}:=\log(2)+\frac{1}{2}\psi(d+\frac{\alpha}{2})+\frac{3-\alpha_{0}}{2-\alpha_{0}}$, 
and it is easy to compute that \eqref{y:eqn} holds
provided that
\begin{equation}
\alpha\leq
1+\frac{-1+\sqrt{1+4y_{0}^{-1}\sqrt{d}}}{2\sqrt{d}}.
\end{equation}
Hence, we conclude that $p(\alpha;d)$ is non-positive and thus
$\frac{\partial}{\partial\alpha}g(\alpha;d)$
is non-positive (by \eqref{g:derivative})
and therefore $g(\alpha;d)$ is decreasing
for any $\alpha\in[1,\alpha'_{0}]$, 
where $\alpha'_{0}:=\min\left(\alpha_{0},1+\frac{-1+\sqrt{1+4y_{0}^{-1}\sqrt{d}}}{2\sqrt{d}}\right)$.
The proof is complete.
\end{proof}


\subsection{Proof of Proposition~\ref{cor:infeas}} \label{ap:infeasibility}

\begin{proof}
Due to our choice of the loss function $f$, the SDEs \eqref{eq:sde_invar_1} and \eqref{eq:sde_invar_2} reduce to Ornstein-Uhlenbeck processes driven by a symmetric $\alpha$-stable L\'{e}vy process. Hence, we can characterize the invariant distributions of the SDEs as follows (see e.g. \cite{Raj2022}):
\begin{align}
    \theta_\infty =^{\text{d}} \mu + \sigma \xi, \qquad \text{and} \qquad \hat{\theta}_\infty =^{\text{d}} \hat{\mu} + \hat{\sigma}\hat{\xi},
\end{align}
for some $\mu, \hat{\mu} \in \mathbb{R}$ and $\sigma, \hat{\sigma} \in \mathbb{R}_+$. Here, $\xi$ and $\hat{\xi}$ are $\mathcal{S}\alpha\mathcal{S}(1)$ distributed (see Section~\ref{sec:bg} for definition) and $=^{\text{d}}$ denotes equality in distribution. 
Now recall that 
$\mu = \mbox{Law}(\theta_\infty)$ and
$\nu = \mbox{Law}(\hat{\theta}_\infty)$, 
and the $p$-Wasserstein metric for one-dimensional distributions is given by,
\begin{align*}
    \mathcal{W}_p^p(\mu,\nu) = \inf_{\gamma \in \Gamma{(\mu,\nu})} \mathbb{E}_{(x,y)\sim \gamma(x,y)} |x-y|^{p},
\end{align*}
where $\Gamma(\mu,\nu)$ is the set of all couplings of $\mu$ and $\nu$. In our case, $x \in \mathbb{R}$ and $y \in \mathbb{R}$. For any coupling  $\gamma^\star\in\Gamma(\mu,\nu)$,
we have
\begin{align*}
    \int_{\mathbb{R} \times \mathbb{R}} |x-y|^p~d\gamma^\star(x,y)
   &= \int_{\mathbb{R} \times \mathbb{R}} \Big[|x-y|^2\Big]^{p/2}~d\gamma^\star(x,y) \\
   &= \int_{\mathbb{R} \times \mathbb{R}} (x^2 + y^2 - 2xy)^{p/2}~d\gamma^\star(x,y) \\
   &\geq \int_{\mathbb{R}_{+} \times \mathbb{R}_{-}} (x^2 + y^2 - 2xy)^{p/2}~d\gamma^\star(x,y) \\
   &\geq \int_{\mathbb{R}_{+} \times \mathbb{R}_{-}} (|x|^p + |y|^p + |2xy|^{p/2})~d\gamma^\star(x,y) \\
   &\geq \int_{\mathbb{R}_{+} \times \mathbb{R}_{-}} |x|^p ~d\gamma^\star(x,y)+ \int_{\mathbb{R}_{+} \times \mathbb{R}_{-}} |y|^p~d\gamma^\star(x,y) \\
   &= C_1 \int_{\mathbb{R}_{+} } |x|^p ~d\mu(x)+ C_2 \int_{ \mathbb{R}_{-}} |y|^p~d\nu(y)= +\infty,
\end{align*}
where $C_1$ and $C_2$ are some finite, positive constants. The last equation comes from the properties of the $\alpha$-stable distribution. 
Since it holds for any $\gamma^{\star}\in\Gamma(\mu,\nu)$, 
we conclude that $\mathcal{W}_p^p(\mu,\nu)=\infty$.
This completes the proof.
\end{proof}

\subsection{Proof of Corollary~\ref{cor:discrete}}

\begin{corollary}[Restatement of Corollary~\ref{cor:discrete}]
Under the assumptions in Theorem~\ref{thm:W:1:re} and Lemma~\ref{lem:discretization:error:re}, we have:

(i) For any $2\leq N\leq\eta^{-1}+1$,
\begin{align}
&\mathcal{W}_{1}\left(\mbox{Law}(\theta_{\eta N}),\mbox{Law}(\hat{\theta}_{\eta N})\right)
\nonumber
\\
&\leq
\left(K_{1}+\rho(X_{n},\hat{X}_{n})K_{2}\right)(2C)(1+C_{0}(1+\Vert w\Vert))\eta^{1+\frac{1}{\alpha}}
+\rho(X_{n},\hat{X}_{n})K_{2}(2C_{0}(1+\Vert w\Vert)+1)\eta
\nonumber
\\
&\qquad 
+e^{L}\left(K_{1}+\rho(X_{n},\hat{X}_{n})K_{2}\right)\left(2C\right)(1+C_{0}(1+\Vert w\Vert))\eta^{\frac{1}{\alpha}}
\nonumber
\\
&\qquad\qquad
+e^{L}\rho(X_{n},\hat{X}_{n})K_{2}(2C_{0}(1+\Vert w\Vert)+1)
+2Q(1+\Vert w\Vert)\eta^{2/\alpha-1},
\end{align} 
and for any $N>\eta^{-1}+1$, 
\begin{align}
&\mathcal{W}_{1}\left(\mbox{Law}(\theta_{\eta N}),\mbox{Law}(\hat{\theta}_{\eta N})\right)
\nonumber
\\
&\leq
\left(K_{1}+\rho(X_{n},\hat{X}_{n})K_{2}\right)\left(2C\right)(1+C_{0}(1+\Vert w\Vert))\eta^{1+\frac{1}{\alpha}} +\rho(X_{n},\hat{X}_{n})K_{2}(2C_{0}(1+\Vert w\Vert)+1)\eta
\nonumber
\\
& \qquad \qquad \qquad 
+\left(C_{1}\lambda^{-1}e^{\lambda}+1\right)
e^{L}\left(K_{1}+\rho(X_{n},\hat{X}_{n})K_{2}\right)\left(2C\right)(1+C_{0}(1+\Vert w\Vert))\eta^{\frac{1}{\alpha}}
\nonumber
\\
&\qquad \qquad \qquad 
+\left(C_{1}\lambda^{-1}e^{\lambda}+1\right)e^{L}\rho(X_{n},\hat{X}_{n})K_{2}(2C_{0}(1+\Vert w\Vert)+1)
+2Q(1+\Vert w\Vert)\eta^{2/\alpha-1}.
\end{align} 

(ii) We have  
\begin{align}
\mathcal{W}_{1}(\mu,\hat{\mu})
\leq
\left(C_{1}\lambda^{-1}e^{\lambda}+1\right)e^{L}\rho(X_{n},\hat{X}_{n})K_{2}\left(2C_{0}+1\right)
+2Q\eta^{2/\alpha-1}.
\end{align}
\end{corollary}

\begin{proof}
Let us prove part (ii) and the proof for part (i) is similar.
It follows directly from Lemma~\ref{lem:discretization:error} 
and Theorem~\ref{thm:W:1} and the inequality:
\begin{equation}
\mathcal{W}_{1}(\nu,\hat{\nu})
\leq
\mathcal{W}_{1}(\nu,\mu)
+\mathcal{W}_{1}(\hat{\nu},\hat{\mu})
+\mathcal{W}_{1}(\mu,\hat{\mu}).
\end{equation}
The proof is complete.
\end{proof}






\section{Technical Lemmas}

In this section, we provide some technical results
that are used in the proofs of main results in Section~\ref{sec:proofs}.
First, we have the following technical result from \citep{chen2022approximation}.

\begin{lemma}[Proposition 2.1. in \citet{chen2022approximation}]\label{lem:uniform}
Under Assumption~\ref{assump:2}, 
$(\theta_{t}^{w})_{t\geq 0}$ and $(\hat{\theta}^{w}_{t})_{t\geq 0}$ admit unique invariant probability measures $\mu$ 
and $\hat{\mu}$ respectively such that
\begin{align}
&\sup_{|f|\leq V}\left|\mathbb{E}[f(\theta_{t}^{w})]-\mu(f)\right|
\leq c_{1}V(w)e^{-c_{2}t},
\qquad\text{for any $t>0$},
\\
&\sup_{|f|\leq V}\left|\mathbb{E}[f(\hat{\theta}_{t}^{w})]-\hat{\mu}(f)\right|
\leq c_{1}V(w)e^{-c_{2}t},
\qquad\text{for any $t>0$},
\end{align}
for some constants $c_{1},c_{2}>0$
where $V(w):=(1+\Vert w\Vert^{2})^{1/2}$
is a Lyapunov function.
In particular, there exists a constant $C_{0}>0$ such that
\begin{align}
&\mathbb{E}\Vert\theta_{t}^{w}\Vert
\leq C_{0}(1+\Vert w\Vert),
\qquad\text{for any $t>0$},
\\
&\mathbb{E}\Vert\hat{\theta}_{t}^{w}\Vert
\leq C_{0}(1+\Vert w\Vert),
\qquad\text{for any $t>0$}.
\end{align}
\end{lemma}

Moreover, we recall the following technical lemma.

\begin{lemma}[Proposition 2.2 in \citet{chen2022approximation}]\label{lem:contraction}
There exist constants $C_{1},\lambda>0$
such that for any $t>0$ and $w,y\in\mathbb{R}^{d}$, we have
\begin{align}
&\mathcal{W}_{1}\left(\mbox{Law}\left(\theta_{t}^{w}\right),\mbox{Law}\left(\theta_{t}^{y}\right)\right)\leq C_{1}e^{-\lambda t}\Vert w-y\Vert,
\\
&\mathcal{W}_{1}\left(\mbox{Law}\left(\hat{\theta}_{t}^{w}\right),\mbox{Law}\left(\hat{\theta}_{t}^{y}\right)\right)\leq C_{1}e^{-\lambda t}\Vert w-y\Vert.
\end{align}
\end{lemma}

Let $P_{t}$ and $\hat{P}_{t}$ denote the Markov semigroups
of $\theta_{t}$ and $\hat{\theta}_{t}$ processes respectively, 
that is, for any bounded function $f:\mathbb{R}^{d}\rightarrow\mathbb{R}$,
\begin{equation}
P_{t}f(x)=\mathbb{E}f(\theta_{t}^{w}),
\qquad
\hat{P}_{t}f(x)=\mathbb{E}f(\hat{\theta}_{t}^{w}).
\end{equation}

We have the following technical lemma from \citep{chen2022approximation}.

\begin{lemma}[Lemma 3.1 in \citet{chen2022approximation}]\label{lem:grad}
For any $h\in\text{Lip}(1)$ and $v,w\in\mathbb{R}^{d}$ and $t\in(0,1]$, we have
\begin{equation}
\Vert\nabla_{v}P_{t}h(w)\Vert
\leq e^{L}\Vert v\Vert,
\qquad
\Vert\nabla_{v}\hat{P}_{t}h(w)\Vert
\leq e^{L}\Vert v\Vert,
\end{equation}
where $L$ is defined in Assumption~\ref{assump:2}.
\end{lemma}

We recall the following technical lemma from \citep{chen2022approximation}.

\begin{lemma}[Lemma 3.2 in \citet{chen2022approximation}]\label{lem:expectation}
There exist constants $C>0$ such that for all $w\in\mathbb{R}^{d}$, $t\geq 0$, we have
\begin{align}
&\mathbb{E}\Vert\theta_{t}^{w}-w\Vert
\leq C(1+\Vert w\Vert)\left(t\vee t^{1/\alpha}\right),
\\
&\mathbb{E}\Vert\hat{\theta}_{t}^{w}-w\Vert
\leq C(1+\Vert w\Vert)\left(t\vee t^{1/\alpha}\right).
\end{align}
\end{lemma}

Next, we state and prove the following key technical lemma.

\begin{lemma}\label{lem:key}
There exist constants $C>0$ such that for all $w\in\mathbb{R}^{d}$, $\eta\in(0,1)$,
$f:\mathbb{R}^{d}\rightarrow\mathbb{R}$ with $\Vert\nabla f\Vert_{\infty}<\infty$, we have
\begin{align}
&\left|P_{\eta}f(w)-\hat{P}_{\eta}f(w)\right|
\nonumber
\\
&\leq
\Vert\nabla f\Vert_{\infty}
\left[\left(K_{1}+\rho(X_{n},\hat{X}_{n})K_{2}\right)2C(1+\Vert w\Vert)\eta^{1+\frac{1}{\alpha}}
+\rho(X_{n},\hat{X}_{n})K_{2}(2\Vert w\Vert+1)\eta\right].
\end{align}
\end{lemma}

\begin{proof}[Proof of Lemma~\ref{lem:key}]
We can compute that
\begin{align*}
&\left|\mathbb{E}\left[f\left(\theta_{\eta}^{w}\right)-f\left(\hat{\theta}_{\eta}^{w}\right)\right]\right|
\\
&=\left|\mathbb{E}\left[f\left(w+\int_{0}^{\eta}\nabla\hat{F}(\theta_{r}^{w},X_{n})dr+L_{\eta}^{\alpha}\right)-f\left(w+\int_{0}^{\eta}\nabla\hat{F}(\hat{\theta}_{r}^{w},\hat{X}_{n})dr+L_{\eta}^{\alpha}\right)\right]\right|
\\
&\leq
\Vert\nabla f\Vert_{\infty}
\mathbb{E}\left\Vert\int_{0}^{\eta}\nabla\hat{F}(\theta_{r}^{w},X_{n})dr-\int_{0}^{\eta}\nabla\hat{F}(\hat{\theta}_{r}^{w},\hat{X}_{n})dr\right\Vert
\\
&\leq
\Vert\nabla f\Vert_{\infty}
\mathbb{E}\int_{0}^{\eta}\left\Vert \nabla\hat{F}(\theta_{r}^{w},X_{n})-\nabla\hat{F}(\hat{\theta}_{r}^{w},\hat{X}_{n})\right\Vert dr
\\
&\leq
\Vert\nabla f\Vert_{\infty}
\mathbb{E}\int_{0}^{\eta}\left(K_{1}\Vert\theta_{r}^{w}-\hat{\theta}_{r}^{w}\Vert
+\rho(X_{n},\hat{X}_{n})K_{2}\left(\Vert\theta_{r}^{w}\Vert+\Vert\hat{\theta}_{r}^{w}\Vert+1\right)\right)dr
\\
&=\Vert\nabla f\Vert_{\infty}
\left[K_{1}\int_{0}^{\eta}\mathbb{E}\Vert\theta_{r}^{w}-\hat{\theta}_{r}^{w}\Vert dr
+\rho(X_{n},\hat{X}_{n})K_{2}\int_{0}^{\eta}\mathbb{E}\left(\Vert\theta_{r}^{w}\Vert+\Vert\hat{\theta}_{r}^{w}\Vert+1\right)dr\right].
\end{align*}
By Lemma~\ref{lem:expectation}, we have
\begin{align*}
\int_{0}^{\eta}\mathbb{E}\Vert\theta_{r}^{w}-\hat{\theta}_{r}^{w}\Vert dr
&
\leq
\int_{0}^{\eta}\mathbb{E}\Vert\theta_{r}^{w}-w\Vert dr
+\int_{0}^{\eta}\mathbb{E}\Vert\hat{\theta}_{r}^{w}-w\Vert dr
\\
&\leq
C(1+\Vert w\Vert)\int_{0}^{\eta}r^{1/\alpha}dr
+C(1+\Vert w\Vert)\int_{0}^{\eta}r^{1/\alpha}dr
\\
&\leq
2C(1+\Vert w\Vert)\eta^{1+\frac{1}{\alpha}}.
\end{align*}
By applying Lemma~\ref{lem:expectation} again, we have
\begin{align*}
\int_{0}^{\eta}\mathbb{E}\left(\Vert\theta_{r}^{w}\Vert+\Vert\hat{\theta}_{r}^{w}\Vert+1\right)dr
&\leq
\int_{0}^{\eta}\mathbb{E}\left(\Vert\theta_{r}^{w}-w\Vert+\Vert\hat{\theta}_{r}^{w}-w\Vert+2\Vert w\Vert+1\right)dr
\\
&\leq
\int_{0}^{\eta}\left(C(1+\Vert w\Vert)r^{1/\alpha}+C(1+\Vert w\Vert)r^{1/\alpha}+2\Vert w\Vert+1\right)dr
\\
&\leq
2C(1+\Vert w\Vert)\eta^{1+\frac{1}{\alpha}}
+(2\Vert w\Vert+1)\eta.
\end{align*}
Hence, we conclude that
\begin{align*}
&\left|\mathbb{E}\left[f(\theta_{\eta}^{w})-f(\hat{\theta}_{\eta}^{w})\right]\right|
\\
&\leq
\Vert\nabla f\Vert_{\infty}
\left[\left(K_{1}+\rho(X_{n},\hat{X}_{n})K_{2}\right)\left(2C\right)(1+\Vert w\Vert)\eta^{1+\frac{1}{\alpha}}
+\rho(X_{n},\hat{X}_{n})K_{2}(2\Vert w\Vert+1)\eta\right].
\end{align*}
This completes the proof.
\end{proof}


\begin{lemma}[Restatement of Lemma~\ref{lem:discretization:error} (Theorem 1.2. in \citet{chen2022approximation})]\label{lem:discretization:error:re}
Let $\mu_{t}$ and $\hat{\mu}_{t}$ denote 
the distributions of continuous-time $\theta_{t}$ and $\hat{\theta}_{t}$
and $\mu$ and $\hat{\mu}$ denote the distributions
of continuous-time $\theta_{\infty}$ and $\hat{\theta}_{\infty}$.
Moreover, let $\nu_{k}$ and $\hat{\nu}_{k}$ denote 
the distributions of discrete-time $\theta_{k}$ and $\hat{\theta}_{k}$
and $\nu$ and $\hat{\nu}$ denote the distributions
of discrete-time $\theta_{\infty}$ and $\hat{\theta}_{\infty}$.
Assume the dynamics start at $w$ at time $0$.
Let $m$, $L$
be as in Assumption~\ref{assump:2}. 

Then, there exists some constant $Q$ (that may depend on $B,m,K,L,M$ from Assumption~\ref{assump:2}) such that
the followings hold.

(i) For every $N\geq 2$ and $\eta<\min\{1,m/(8L^{2}),1/m\}$, one has
\begin{align}
&\mathcal{W}_{1}(\mu_{N\eta},\nu_{N})
\leq Q(1+\Vert w\Vert)\eta^{2/\alpha-1},
\\
&\mathcal{W}_{1}(\hat{\mu}_{N\eta},\hat{\nu}_{N})
\leq Q(1+\Vert w\Vert)\eta^{2/\alpha-1}.
\end{align}

(ii) For every $\eta<\min\{1,m/L^{2},1/m\}$, one has
\begin{align}
&\mathcal{W}_{1}(\mu,\nu)
\leq Q\eta^{2/\alpha-1},
\\
&\mathcal{W}_{1}(\hat{\mu},\hat{\nu})
\leq Q\eta^{2/\alpha-1}.
\end{align}
\end{lemma}


\end{document}